\RequirePackage{fix-cm}
\documentclass[twocolumn]{svjour3} 
\smartqed  
\usepackage{graphicx}

\usepackage{amsmath}
\usepackage{mathrsfs}
\usepackage{amsfonts}
\usepackage{amssymb}            
\usepackage{mathtools}

\usepackage{amsthm}			    
\usepackage[font=footnotesize]{subcaption}		
\usepackage{color}			    
\usepackage{xspace}
\usepackage[ruled,vlined,linesnumbered]{algorithm2e} 
\usepackage{algorithmic}        
\usepackage{pgfplots}
\usepackage{natbib}
\usepackage{url}
\usepackage{siunitx}
\usepackage[shortlabels]{enumitem} 

\definecolor{blues1}{RGB}{198, 219, 239}
\definecolor{blues2}{RGB}{158, 202, 225}
\definecolor{blues3}{RGB}{107, 174, 214}
\definecolor{blues4}{RGB}{49, 130, 189}
\definecolor{blues5}{RGB}{8, 81, 156}
\pgfplotscreateplotcyclelist{colorbrewer-blues}{
{blues1},
{blues2},
{blues3},
{blues4},
{blues5},
}
\pgfplotsset{select coords between index/.style 2 args={
    x filter/.code={
        \ifnum\coordindex<#1\fi
        \ifnum\coordindex>#2\fi
    }
}}
\pgfplotsset{compat=1.11,
    /pgfplots/ybar legend/.style={
    /pgfplots/legend image code/.code={%
       \draw[##1,/tikz/.cd,yshift=-0.25em]
        (0cm,0cm) rectangle (3pt,0.8em);},
   },
}
\usepackage{hyperref}		    
\usepackage{xcolor}
\hypersetup{
    colorlinks,
    linkcolor={red!50!black},
    citecolor={blue!50!black},
    urlcolor={blue!50!black}
}
\usepackage{etoolbox}
\makeatletter
\patchcmd{\NAT@citex}
  {\@citea\NAT@hyper@{%
     \NAT@nmfmt{\NAT@nm}%
     \hyper@natlinkbreak{\NAT@aysep\NAT@spacechar}{\@citeb\@extra@b@citeb}%
     \NAT@date}}
  {\@citea\NAT@nmfmt{\NAT@nm}%
   \NAT@aysep\NAT@spacechar\NAT@hyper@{\NAT@date}}{}{}
\patchcmd{\NAT@citex}
  {\@citea\NAT@hyper@{%
     \NAT@nmfmt{\NAT@nm}%
     \hyper@natlinkbreak{\NAT@spacechar\NAT@@open\if*#1*\else#1\NAT@spacechar\fi}%
       {\@citeb\@extra@b@citeb}%
     \NAT@date}}
  {\@citea\NAT@nmfmt{\NAT@nm}%
   \NAT@spacechar\NAT@@open\if*#1*\else#1\NAT@spacechar\fi\NAT@hyper@{\NAT@date}}
  {}{}
\makeatother

\newcommand{\halfpage}{0.47\textwidth}
\newcommand{\thirdpage}{0.31\textwidth}
\newcommand{\quarterpage}{0.21\textwidth}

\newcommand{\defeq}{\vcentcolon=}

\newcommand{\freespace}{\mathscr{W}}
\newcommand{\astar}{A${}^\star$\xspace} 
\newcommand{\trpmpp}{{\sc trpmpp}\xspace}
\newcommand{\trpmppjr}{{\sc trpmpp} Jr.\xspace}
\newcommand{\pt}[1]{\mathbf{#1}}

\newcommand{\length}[1]{\mathfrak{L}\left[#1\right]}

\newcommand{\verts}[1]{\operatorname{verts}(#1)}

\newcommand{\wfeas}{$\freespace$-feasible motion for an $\ell$-length cable\xspace}

\newcommand{\unitInterval}{\mathbb{I}}
\newcommand{\cStar}{\mathbb{C}^\star\xspace}

\newcommand{\cat}[5]{\operatorname{cat}(#1; #2, #3, #4, #5)}
\newcommand{\cathat}[4]{\widehat{\operatorname{cat}}(\cdot;#1, #2, #3, #4)}
\newcommand{\trpmppTuple}{$(\freespace, O, \pt{r_a}, \pt{r_b}, \pt{d_a}, \pt{d_b}, \ell, c_0)$\xspace}

\newcommand{\frechet}{Fr\'{e}chet}

\newcommand{\deleted}[1]{}

\journalname{Autonomous Robots}
\begin{document}

\title{Motion Planning for a Pair of Tethered Robots
}

\author{Reza H. Teshnizi \and Dylan A. Shell}

\institute{
    R.H. Teshnizi \and D.A. Shell \at
    Distributed AI and Robotics Laboratory,\\
	Department of Computer Science and Engineering,\\
	Texas A\&M University, College Station, TX 77843, USA.\\
    \email{reza.teshnizi@gmail.com}
    \at
	\email{dshell@tamu.edu}
}

\maketitle

\begin{abstract}
Considering an environment containing polygonal obstacles,
we address the problem of planning motions for a pair of planar robots connected to one another via a cable of limited length. 
Much like prior problems with a single robot connected via a cable to a fixed base, straight line-of-sight visibility plays an important role. The present paper shows how the reduced visibility graph provides a natural discretization and captures the essential topological considerations very effectively for the two robot case as well.
Unlike the single robot case, however, the bounded cable length introduces considerations around coordination (or equivalently, when viewed from the point of view of a centralized planner, relative timing) that complicates the matter. 
Indeed, the paper has to introduce a rather more involved formalization than prior single-robot work in order to establish the core theoretical result---a theorem permitting the problem to be cast as one of finding paths rather than trajectories.
Once affirmed, the planning problem reduces to a straightforward graph search with an elegant representation of the connecting cable, demanding only a few extra ancillary checks that ensure sufficiency of cable to guarantee feasibility of the solution.
We describe our implementation of \astar search, and report experimental results.
Lastly, we prescribe an optimal execution for the solutions provided by the algorithm.
\end{abstract}

\section{Introduction}\label{sec:intro}

\begin{figure}[t]
    \centering
    \includegraphics[width=\halfpage]{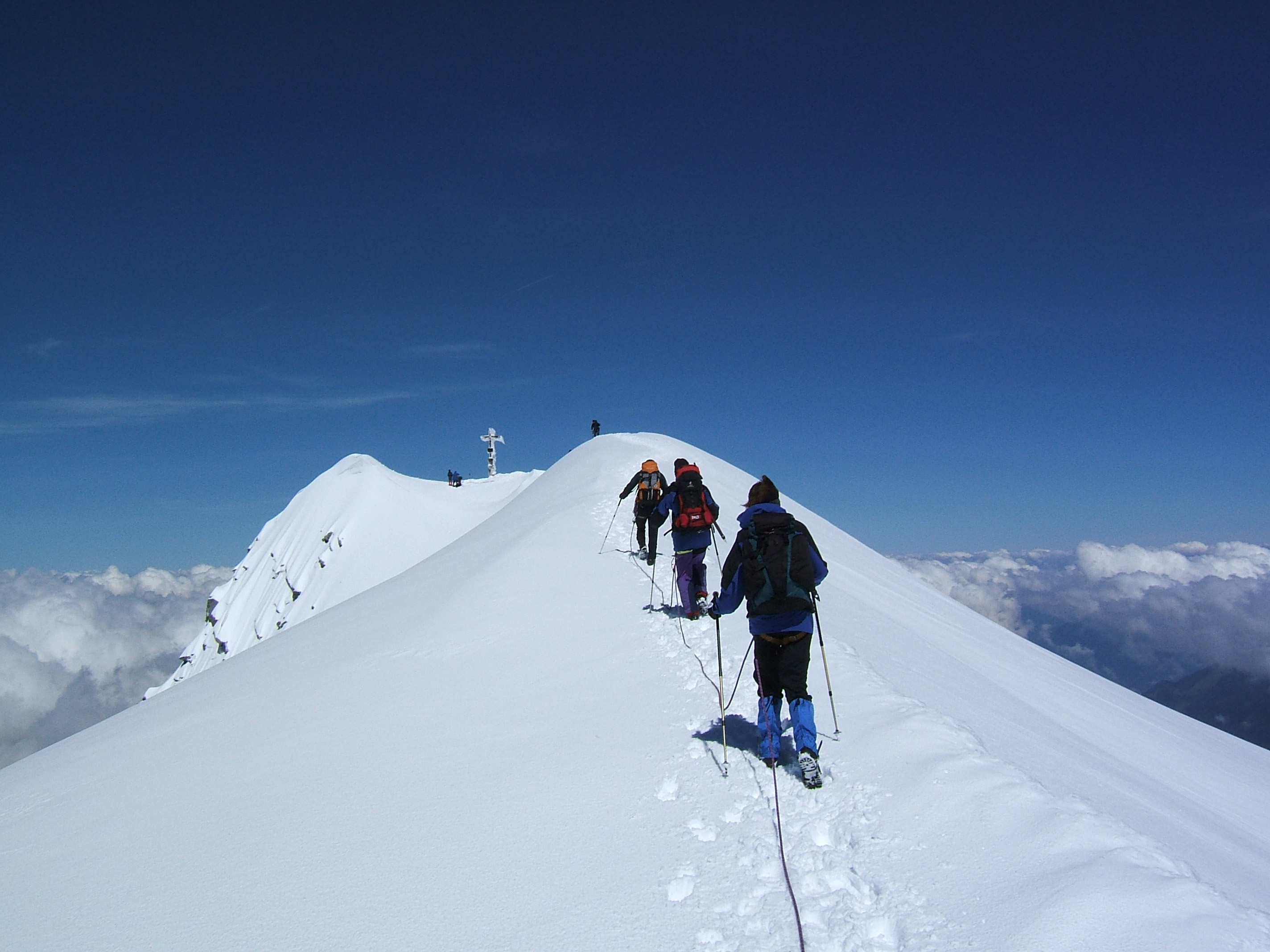}
    \caption{Rope team at a ridge nearing a summit. (Source: Wikimedia Commons)}
    \label{fig:roped-team}
\end{figure}

In recent years, a variety of techniques have been developed to plan motions for a tethered mobile robot~\citep{Igarashi2010-homotopic, teshnizi2014tethered, Kim2014-tethered-robot}.
A tether can be useful as a conduit for power or communication but 
the main motivating application for 
robotic tethers 
is in navigation of rovers in extreme terrain, where the tether can help provide physical security.
Examples of robotic rovers equipped in this way include TRESSA~\citep{Huntsberger2007-tressa}, Axel and DuAxel~\citep{Nesnas2012-axel-duaxel}, vScout~\citep{Stenning2015-vscout}, and TReX~\citep{McGarey2018-trex}, among others.
Humans deal with extreme terrain too. A common practice among mountaineers, as a measure of protection against falling, is to form a group that can move together while the members are roped to one another. This forms what is referred to as a \textit{rope team}~\citep{Gooding2014-snow-climbing} (Fig.~\ref{fig:roped-team}).
From the perspective of motion planning, one might interpret a rope team as a practical scenario in which a tethered robot's base is itself subject to motion.
A prominent robotic example of comparable operation is DuAxel, where two Axel rovers are connected to a central module.
DuAxel is designed to work as a \textit{mother-daughter ship}, having one Axel remain stationary with the central module while the other explores the terrain.
Indeed, enabling both rovers to be deployed simultaneously may benefit both agents and improve the versatility of the design.

One key to efficient solution of planning problems in finding a suitable representation, ideally one that expresses constraints and is amenable to adaptation and generalization to various requirements.
Much like our earlier work \citep{teshnizi2014tethered},
we take advantage of the properties of reduced visibility graph \citep{Latombe1991}.
Namely, we show certain characteristics of straight line motions enable us to find solutions to tethered pair problem with minimal book keeping.
Unlike that work, however, the proofs proceed without needing to supply geometric or topological insight from the
structure of configuration space (c-space).
We do discuss, in Section~\ref{sec:cspace-structure}, some of our understanding of the c-space of this problem.

The theory established in this work shows that we shall not be concerned with the intermediate state of the cable:
as long as we can achieve a goal configuration that is permitted by the length of the available cable, one can provide a planner to execute the motions that will transform the initial cable configuration to its final configuration.
This would be trivial if the constraints involved were purely topological or we considered only a single robot.
But with a finite cable, what one robot uses is unavailable to the other, coupling their motions in space and time.
On the basis of this theory, we introduce a tree data structure that represent different cable configurations up to homotopy.
Each branch in the tree down to a certain node is a representation of the shortest path required for the tethered pair to arrive at that node's cable configuration.
We have implemented \astar search to expand the search tree and find an optimal solution while keeping track of the cable's configuration up to homotopy.
Lastly, we detail how to turn paths  produced by the planner 
into trajectories for the motions
that also minimize a time cost.

\section{Related Work}\label{sec:related-works}

We are interested in what is perhaps the most natural motion planning question for a tethered pair of robots, namely finding 
paths to take a pair of tethered robots from some initial configuration to a goal one, never violating a bound on the tether's length throughout the motion.
Although motion planning for a single tethered robot has been extensively studied~\citep{shnaps-RSS-13-online-coverage-by-tethered-robot,teshnizi2014tethered, Kim2014-tethered-robot, Teshnizi2016-stiff-tether, McCammon2017-underwater-tethered, McGarey2017-tslam},
the literature reports comparatively little work on motion planning problems involving pairs of robots tethered to one another.

A notable exception is that pairs of conjoined robots have been studied for purposes of object manipulation.
\citet{kim2013-rss-separate-and-manipulate-objects} studied object separation using a pair of robots connected by a cable.
Though superficially similar to our problem, as it involves the motion of a mutually connected pair of robots, the separation problem imposes quite a different set of constraints to a shortest path planning problem.
For object separation the solution is only required to satisfy a homotopy requirement, allowing the robots to choose any arbitrary goal in the workspace that can satisfy such constraint.
Consequently, in that work, the two robots move to the boundaries of the workspace.
(This assumption also helps distinguish between separating versus non-separating configurations elegantly and concisely.) As 
\citeauthor{kim2013-rss-separate-and-manipulate-objects}
are addressing a problem where the goal is specified topologically and they are not concerned with a cable of finite length, several of the complications we tackle do not arise in their setting.

More recently, \citet{kim2018tail} demonstrated a physical multi-robot system in which robots can make or break tether connections dynamically; a planner exploits this capability to find ways in which a robot team can manipulate objects efficiently, either with single robots operating concurrently, or as coupled pairs, as  called for by the particular problem instance.  A part of that problem is combinatorial and the work uses a sampling-based method, the present work being distinguished from that work on both fronts.

Our own prior work on finding short paths for a pair of tethered robots~\citep{teshnizi-2016-tethered-paris-workshop}, attempted to use the solution to a single robot problem as an algorithmic building block.
That approach was devised primarily as a means to build intuition for the topology of the four dimensional configuration space;
the algorithm outlined in that work is inadequate, being neither a solution to the complete problem nor one that always yields optimal paths.
Toward the end of the paper, we will return to this matter of visualizing the c-space in light of the correct and complete algorithm that is in this paper.

\section{Problem statement} \label{sec:problem-stat}

Let $O =\{ o_1, o_2, o_3, \dots, o_n \} $ be a (possibly empty) set of pairwise disjoint polygonal obstacles with vertices (written $\verts{O}$) in $\mathbb{R}^2$, with boundary curves $\delta o_1, \delta o_2,$ $\dots, \delta o_n$, respectively.
We will write $\freespace$ for the 
free space, so $\freespace = (\mathbb{R}^2 \backslash \bigcup_{i = 1}^n o_i) \cup (\bigcup_{i = 1}^n \delta o_i)$.
Further, let robots $a$ and $b$ be two unoriented points in $\freespace$ that are connected to one another via a cable of finite length.
For conciseness, let $\unitInterval = [0,1] $ be the unit interval.

\begin{definition}[\trpmpp]\label{def:problem}
    The tethered robot pair motion planning problem (\trpmpp) is a tuple,\\
    \trpmppTuple, wherein:
    \begin{itemize}
    	\item $\freespace$ is the free space,
    	\item $ O $ is the set of obstacles,
    	\item $\pt{r_a} \in \freespace$ is the initial position of $a$,
    	\item $\pt{r_b} \in \freespace$ is the initial position of $b$,
    	\item $\pt{d_a} \in \freespace$ is the goal or destination of $a$,
    	\item $\pt{d_b} \in \freespace$ is the goal or destination of $b$,
    	\item $\ell\in \mathbb{R}^{+} $ is the length of the tether,
    	\item $c_0: \unitInterval \to \freespace$ is the initial arrangement of the cable in $\freespace$, where
    	    $c_0(0) = \pt{r_a}$, and
    	    $c_0(1) = \pt{r_b}$, and\footnote{We find it convenient to use $\length{f}$ to denote arc length of functions $f: \unitInterval \to \freespace$ throughout this work.}
    	    $$ \length{c_0(s)} \defeq  \int_\unitInterval c_0(s) \mathop{ds} \leq \ell. $$
    	    
    \end{itemize}
    A solution to \trpmpp \trpmppTuple is a pair of paths $(\tau_a, \tau_b)$ in which:
        \begin{itemize}
            \item $ \tau_a: \unitInterval \to \freespace$ where $ \tau_a(0) = \pt{r_a} $ and $ \tau_a(1) = \pt{d_a} $, and
            \item $ \tau_b: \unitInterval \to \freespace$ where $ \tau_b(0) = \pt{r_b} $ and $ \tau_b(1) = \pt{d_b} $, and
            \item for $\tau_a$ and $\tau_b$, there exists a \wfeas, as formalized next.
        \end{itemize}
\end{definition}

\begin{definition}[\wfeas] \label{def:wfeas}
    In the closed subset of the plane $\freespace \subseteq \mathbb{R}^2$,
    for two paths $\tau_a, \tau_b: \unitInterval \to \freespace$ and length $\ell\in\mathbb{R}^+$,
    a function  $ c: \unitInterval \times \unitInterval \to \freespace$ is called a \emph{\wfeas} if and only if the following three conditions hold:
    \begin{enumerate}
        \item[\sc c-i]\label{def:wfeas-c1} (The cable connects the robots)
            $$ \forall s \in \unitInterval: c(s, 0) = \tau_a(s) \textrm{ and } c(s, 1) = \tau_b(s), $$
        \item[\sc c-ii]\label{def:wfeas-c2} (Cable has bounded length)
            $$ \forall s \in \unitInterval: \length{c_s(x)} \leq \ell, \textrm{ where } c_s(x) \defeq c(s,x),$$
        \item[\sc c-iii]\label{def:wfeas-c3} (Continuity)
            $ c $ is continuous with respect to the induced topologies.
    \end{enumerate}
\end{definition}

A visual example that serves as a quick summary of the preceding definitions appears in Fig.~\ref{fig:trpmpp-solution}.

\begin{definition}[Distance optimality]\label{def:optimal-solution}
    A solution pair $(\tau^\star_a, \tau^\star_b)$ for \trpmpp $(\freespace, O, \pt{r_a}, \pt{r_b}, \pt{d_a}, \pt{d_b}, \ell, c_0)$ is called \emph{distance optimal} if all other solutions $(p, q)$ have 
    $$\max\left(\length{\tau^\star_a}, \length{\tau^\star_b}\right) \leq \max\left(\length{p}, \length{q}\right).$$ 
\end{definition}

This particular optimality metric models minimum energy consumption between the two robots.
We will present a method that finds a distance optimal solution to a \trpmpp.
In Section~\ref{sec:optimal-control} we show that solutions under this distance metric can yield optimal solutions under a time-based criterion too.
    
\section{The solution concept} \label{sec:solution}

\begin{figure}[t]
    \centering
    \includegraphics[scale=0.8]{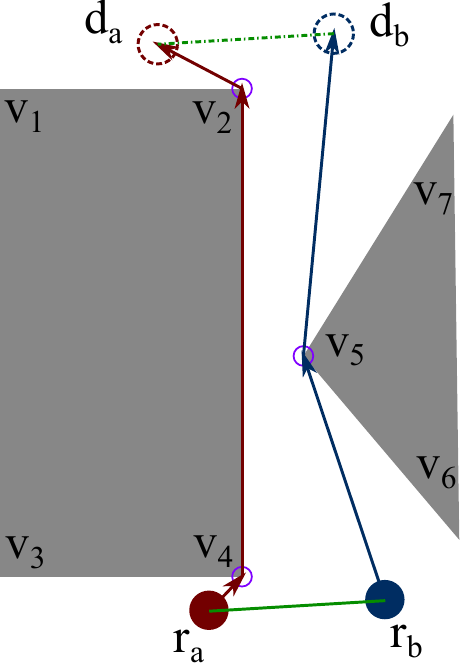}
    \caption{
        The green line segment is the initial cable configuration and the green dashed line segment is the final cable configuration.
        The arrows show the distance optimal path for each robot.
    }
    \label{fig:trpmpp-solution}
\end{figure}

We prove, first, that it suffices to have robots $a$ and $b$ move on a straight line from one vertex to another in the reduced visibility graph (RVG). (Though our interest is in characterizing the solution set, this is potentially also good news in terms of the sensing technology needed for the robots to be able to execute optimal plans, cf.~\citet{Tovar2007-distance-optimal}.)
The following theorem and widely known corollary indicate that, for a single robot, the shortest path between two points in $\freespace$ is a  concatenation of line segments that are edges of the RVG.

\begin{theorem}
    There exists a semi-free path between any two given points $\pt{p}$ and $\pt{q}$
    if and only if there exists a simple polygonal line $T$ lying in $\freespace$
    whose endpoints are $\pt{p}$ and $\pt{q}$, and such that $T$'s vertices are in $\verts{O}$.
    \citep{Latombe1991, DeBerg2008-comp-geometry}
\end{theorem}

\begin{figure*}[t]
    \centering
    \begin{subfigure}[t]{\thirdpage} \centering
        \includegraphics[height=4cm]{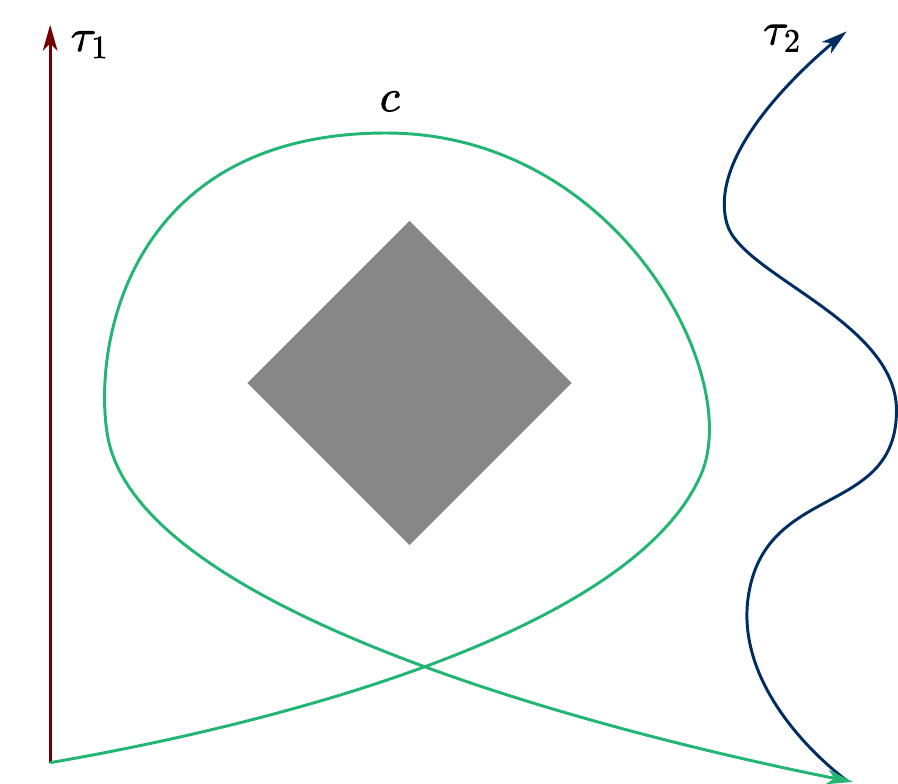}
        \caption{
            The three paths used in the $\operatorname{cat}$ operator.
        }
    \end{subfigure}\hfill%
    \begin{subfigure}[t]{\thirdpage} \centering
        \includegraphics[height=4cm]{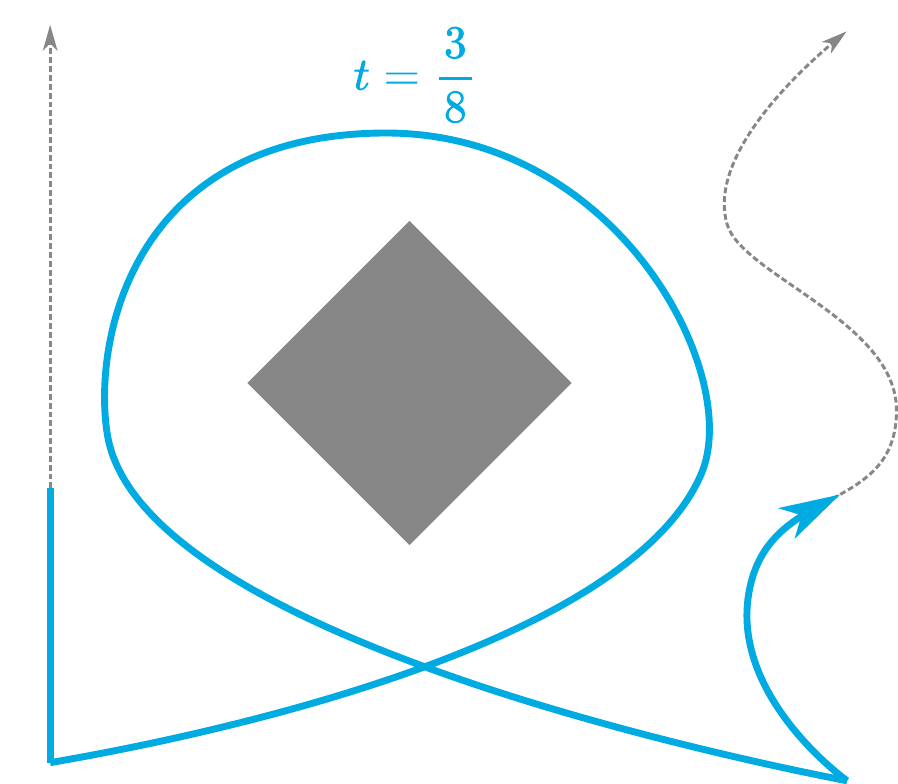}
        \caption{
            Choosing a value for $t$ gives a curve.
        }
    \end{subfigure}\hfill%
    \begin{subfigure}[t]{\thirdpage} \centering
        \includegraphics[height=4cm]{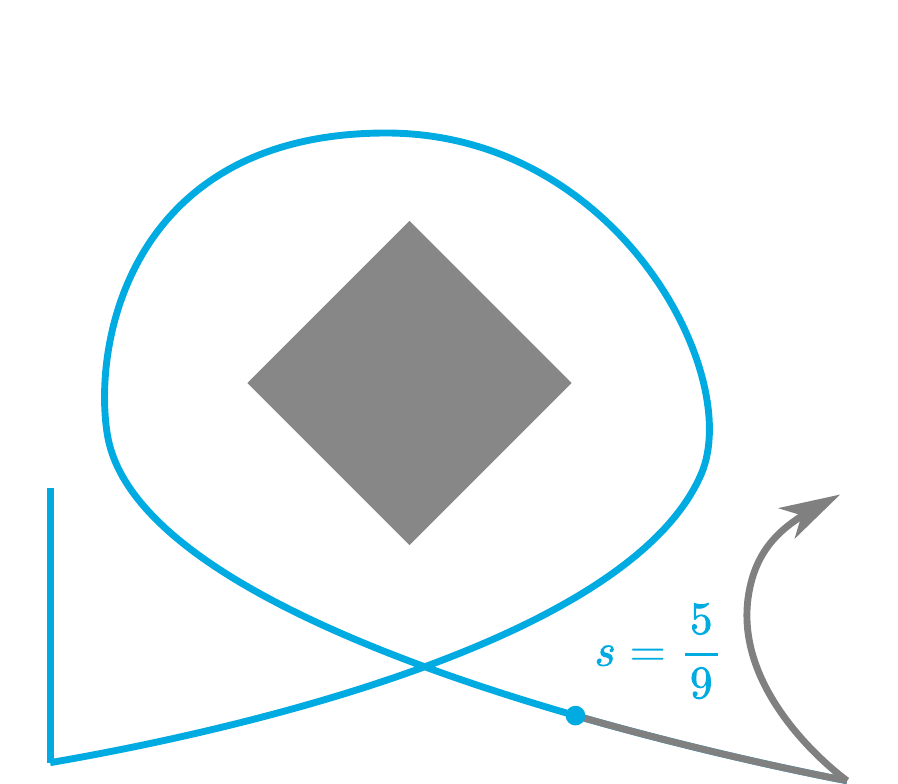}
        \caption{
            Choosing a value for $s$ gives a point of the curve.
        }
    \end{subfigure}
    \caption{A visual explanation of how three paths are connected via the $\operatorname{cat}$ operator.}\label{fig:cat}
\end{figure*}

More useful for us the statement below which, though only stated informally by
\citet{Lavalle1999-planning-algs}, follows directly the previous result.

\begin{corollary}\label{thm:shortest-path-on-rvg}
    The shortest path for a robot from one point to another in a subset of $\mathbb{R}^2$ can be found by searching the shortest path roadmap or reduced visibility graph.
\end{corollary}


The list of RVG vertices $\pi_a = (v^0_a, v^1_a, \dots, v^{n-1}_a, v^n_a)$ connecting initial and destination positions, \textit{i.e.}, with $\pt{r_a}=v^0_a$ and $\pt{d_a}=v^n_a$ , is easily turned into a curve, $\tau_a(s)$, by joining line segments connecting $v^{i-1}_a$ to $v^{i}_a$ sequentially, head to tail, and parameterizing appropriately via $\unitInterval$.
Hence, a pair of sequences of RVG vertices for the two robots, $(\pi_a, \pi_b)$, suffices to give a pair of paths $(\tau_a, \tau_b)$.

Although the preceding classical results hold for individual robots,
when two robots are constrained such that the action of one limits the actions of the other
(as in the case of a tether of finite length),
then it is less clear that the discrete structure of the RVG encodes an optimal solution.
One might conceive, in the two robot setting, one robot deviating from visibility edges in order to enable the other to move.
A main result of the paper (Theorem~\ref{thm:dist-optimality}), and the basis for the algorithm we present, establishes that no such deviations are necessary. 
Indeed, the RVG still suffices to find optimal solutions.
More formally, let $\Pi_a$ be the set of all paths from $\pt{r_a}$ to $\pt{d_a}$ and let $\Pi_b$ be the set of all paths from $\pt{r_b}$ to $\pt{d_b}$ in the RVG.
Proof of Theorem~\ref{thm:dist-optimality},
the fact that $(\tau^\star_a, \tau^\star_b) \in \Pi_a \times \Pi_b$,
is via several steps. 
The largest single intermediate step is in establishing Lemma~\ref{thm:rvg-paths-are-convex}, on the following page,
but it requires a few definitions first. 

To start, in what follows,  we can think of an always taut tether.
This is without loss of generality because:
(1) tightening a tether is a continuous operation, so it preserves homotopy;
(2) if the cable length constraint is satisfied for a taut tether, it must be for others as well.
The reader should bear in mind that the taut tether is merely a special representative of the homotopy class of tethers. It is helpful to have an operator to give this representative:

\begin{definition}[Tightening Operator]\label{def:tightening-op}
    Given a path $\alpha: \unitInterval \to \freespace$, we define the operator $\hat{\cdot}: (\unitInterval \to \freespace) \to (\unitInterval \to \freespace)$
    such that $\hat{\alpha}$ is the (unique) shortest path in the homotopy class of~$\alpha$.
\end{definition}

Practically, the classical algorithm of \citet{Hershberger1994} is used to obtain the shortest path homotopic to a given path.
The tightening operator will be used, in what follows, on paths and also on cable configurations. For the latter,  the following definition helps.

\begin{definition}[Connecting Paths via a Cable]\label{def:cat}
    Given two paths $\tau_1: \unitInterval \to \freespace$ and $\tau_2: \unitInterval \to \freespace$, and a cable configuration $c: \unitInterval \to \freespace$,
    we define a function that concatenates $\tau_1$, $\tau_2$, and $c$.
    Let
    $$
        \cat{s}{t}{\tau_1}{\tau_2}{c} \defeq 
        \begin{cases}
            \tau_1(t - 3 s t) &\quad 0 \leq s \leq \frac{1}{3} \\
            c(3(s - \frac{1}{3})) &\quad \frac{1}{3} < s < \frac{2}{3} \\
            \tau_2(3(s - \frac{2}{3})t) &\quad \frac{2}{3} \leq s \\
        \end{cases}\text{.}
    $$
    Then, for any fixed $t \in \unitInterval$, argument 
    $s \in \unitInterval$ parameterizes a new curve comprising those portions of $\tau_1$ and $\tau_2$ up to $t$ connected via $c$.
\end{definition}

Since both $s$ and $t$ are curve parameters, the definition of $\operatorname{cat}$ can appear abstruse at first.
Fig.~\ref{fig:cat} helps by illustrating how three given paths are joined via $\operatorname{cat}$ .

Definition~\ref{def:problem}, describing the planning problem, requires paths for which a feasible trajectory exists.
The connection between paths and trajectories is, of course, a timing.
Thus we need the following concept:

\begin{definition}[Re-parameterization]\label{def:reTiming}
\\A \emph{re-parameterization} is a monotonically increasing, continuous function 
$r: \unitInterval  \to \unitInterval$  with $r(0) = 0$ and $r(1) = 1$.
A pair $(r_1, r_2$) is a \emph{re-parameterization pair} if both functions, $r_1$ and $r_2$, are re-parameterizations.
\end{definition}

With the preceding scaffolding, we can next give a definition that is valuable in helping to identify the existence of a feasible trajectory given paths.
The notation $\circ$ is for function composition, i.e.,
$\tau_1 \circ r_1(t) = \tau_1(r_1(t))$ and $\tau_2 \circ r_2(t) = \tau_2(r_2(t))$.

\begin{definition}\label{def:cStar}
    Given two paths $\tau_1: \unitInterval \to \freespace$ and $\tau_2: \unitInterval \to \freespace$, and a cable configuration $c: \unitInterval \to \freespace$,
    let
    \begin{align*}
        &\cStar(\tau_1, \tau_2, c) \defeq \\
        &\min_{\substack{(r_1,r_2)~\text{over all} \\ \text{re-parameterization pairs}}}\left(\max_{t \in \unitInterval}\,\length{{\cathat{t}{\tau_1 \circ r_1}{\tau_2 \circ r_2}{c}}}\right).    
    \end{align*}
\end{definition}

Informally, $\cStar$ gives the shortest cable that permits one to execute $\tau_1$ and $\tau_2$.
The definition of $\cStar$ resembles the \frechet~distance~\citep{Ewing1969}.
The key difference is that $\cStar$ takes in a cable configuration $c$ and obtains the measure subject to the path connecting the two curves
being constrained to a particular homotopy class.
That homotopy class is prescribed via cable configuration $c$. (See also discussion in Section~\ref{sec:cspace-structure}.)

The intuition which connects with the main result is that any solution $(\tau_a, \tau_b)$ to a given \trpmpp must have a $\cStar(\tau_a, \tau_b, c_0) \leq \ell$.
We would like to establish that any optimal solution to a given \trpmpp
can be related to some solution on the RVG (\textit{i.e.}, in $\Pi_a \times \Pi_b$) that,
while being distance optimal, also abides by the conditions imposed by cable.
Corollary~\ref{thm:dist-optimality} states this formally, but we need the following lemmas to get there. 

\begin{figure*}[t]
    \centering
    \begin{subfigure}[t]{0.47\textwidth} \centering
        \includegraphics[height=3.2cm]{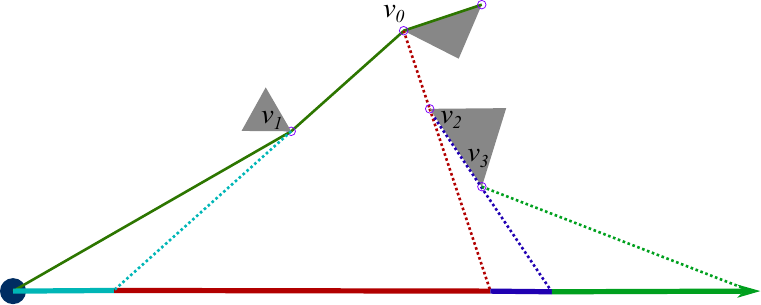}
        \caption{
            A scenario illustrating the cable consumption function.
            Originally the cable is in contact with $v_0$ and $v_1$.
            As the robot travels towards its destination (to the right), it will release contact with $v_1$. It will then make contact with $v_2$ and $v_3$, respectively.
            The dotted lines represent places in which cable events occur.
        }\label{fig:convex-scenario}
    \end{subfigure}\hfill%
    \begin{subfigure}[t]{0.22\textwidth}
        \begin{tikzpicture}[
            declare function={
                f0(\x) = and(\x>-30, \x<=-10+15) * ((4 * sqrt(2)) + sqrt(10) + sqrt(6^2 + (4 + \x - 15)^2));
                f1(\x) = and(\x>-10+15, \x<=(10/3)+15) * (sqrt(10) + sqrt(10^2+(\x - 15)^2));
                f2(\x) = and(\x>(10/3)+15, \x<=(17/3)+15) * (sqrt(10) + sqrt(10) + sqrt(7^2 + (\x - 1 - 15)^2));
                f3(\x) = (\x>(17/3)+15) * (sqrt(10) + sqrt(10) + sqrt(13) + sqrt(4^2+(\x - 3 - 15)^2));
            }
        ]
            \begin{axis}[
                axis lines=middle,
                xlabel={Curve parameter}, xmin=0, xmax=32,
                ylabel={Consumed cable}, ymin=0,
                ylabel style={rotate=90, at={(-0.15,0.5)}},
                ticks=none,
                domain=0:30,
                height=4.5cm
            ]
                \addplot[color=cyan, domain=0:5, samples=20, thick]{f0(x)};
                \draw [cyan,thick,densely dotted] (5,0) -- (5,17.3);
                \addplot[color=red, domain=5.001:18.333, samples=20, thick]{f1(x)};
                \draw [red,thick,densely dotted] (18.333,0) -- (18.333,13.7);
                \addplot[color=blue, domain=18.334:20.666, samples=20, thick]{f2(x)};
                \draw [blue,thick,densely dotted] (20.666,0) -- (20.666,14.7);
                \addplot[color=black!30!green, domain=20.667:30, samples=20, thick]{f3(x)};
                \draw [black!30!green,thick,densely dotted] (30,0) -- (30,22.6);
            \end{axis}
        \end{tikzpicture}
        \caption{Length of consumed cabled as a function of curve parameter.}\label{fig:convex-cable}
    \end{subfigure}\hfill%
    \begin{subfigure}[t]{0.22\textwidth}
        \begin{tikzpicture}[
            declare function={
                f0(\x) = and(\x>-30, \x<=-10+15) * ((\x - 11) / sqrt(6^2 + (4 + \x - 15)^2));
                f1(\x) = and(\x>-10+15, \x<=(10/3)+15) * ((\x - 15) / sqrt(10^2+(\x - 15)^2));
                f2(\x) = and(\x>(10/3)+15, \x<=(17/3)+15) * ((\x - 16) / sqrt(7^2 + (\x - 1 - 15)^2));
                f3(\x) = (\x>(17/3)+15) * ((\x - 18) / sqrt(4^2+(\x - 3 - 15)^2));
            }
        ]
            \begin{axis}[
                axis lines=middle,
                xlabel={Curve parameter}, xmin=0, xmax=32,
                ylabel={Derivative},
                ylabel style={rotate=90, at={(-0.15,0.5)}},
                ticks=none,
                domain=0:30,
                height=4.5cm
            ]
                \addplot[color=cyan, domain=0:5, samples=20, thick]{f0(x)};
                \draw [cyan,thick,densely dotted] (5,-0.7) -- (5,0);
                \addplot[color=red, domain=5.001:18.333, samples=20, thick]{f1(x)};
                \draw [red,thick,densely dotted] (18.333,0) -- (18.333,0.3);
                \addplot[color=blue, domain=18.334:20.666, samples=20, thick]{f2(x)};
                \draw [blue,thick,densely dotted] (20.666,0) -- (20.666,0.55);
                \addplot[color=black!30!green, domain=20.667:30, samples=20, thick]{f3(x)};
                \draw [black!30!green,thick,densely dotted] (30,0) -- (30,0.95);
            \end{axis}
        \end{tikzpicture}
        \caption{Derivative of length of consumed cabled per change in curve parameter.}\label{fig:convex-derivative}
    \end{subfigure}
    \caption{
        An illustrative example showing how, for a straight line motion, 
        the cable consumption is a convex function.
        As the robot (in Fig.~\ref{fig:convex-scenario}) makes its way towards the destination,
        the cable will make or release contact with elements of $\verts{O}$.
        These events change the derivative of the consumed cable.
        Color coding of continuous motion helps show how the function's pieces are separated by the occurrence of discrete cable events.
        Fig.~\ref{fig:convex-cable} and \ref{fig:convex-derivative} show the function and its derivative, respectively.
    }\label{fig:convexity}
\end{figure*}
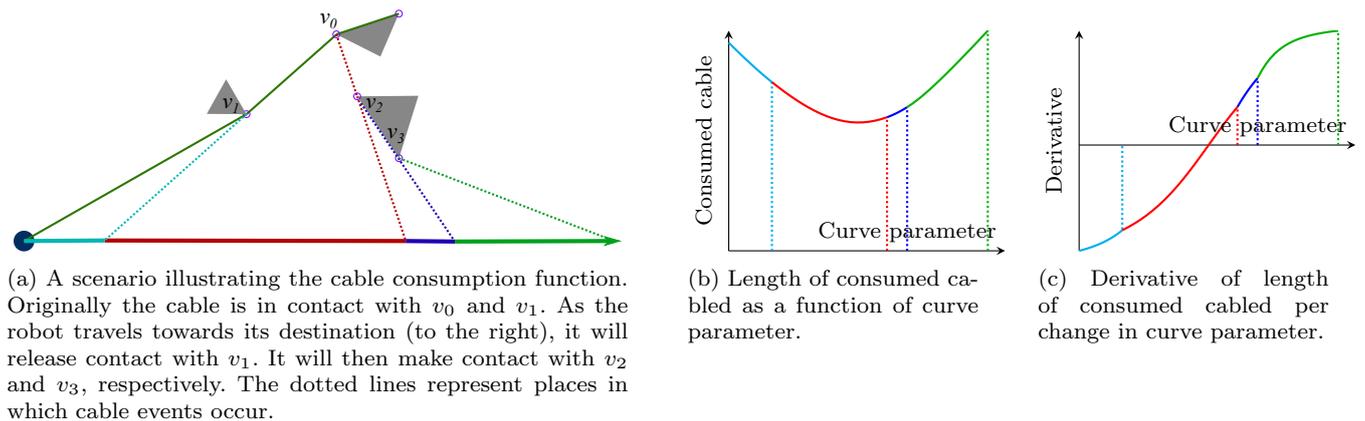

\begin{lemma}\label{thm:cstar-is-less-than-ell}
    Path pair $(\tau_a, \tau_b)$ is a solution to a given
    \trpmpp, iff $\cStar(\tau_a, \tau_b, c_0) \leq  \ell$.
\end{lemma}
\begin{proof}
    For a contradiction assume $\cStar(\tau_a, \tau_b, c_0) > \ell$
    and reach a contradiction with \textsc{c-ii} 
    in Definition~\ref{def:wfeas}.
\end{proof}

The following lemma establishes an important fact about consumption of cable as the robots move on short paths. 
The result states that the consumption is a convex function, which leads to a practical way to verify that a pair of paths do not violate the limits imposed by the bounded cable, i.e., it suffices to check the lengths at the two end points.

\begin{lemma}\label{thm:rvg-paths-are-convex}
    Let $\widehat{\tau}_a$ and $\widehat{\tau}_b$, being the shortest paths in their respective homotopy classes, be the prescribed paths for robots $a$ and $b$.
    Then the length of the consumed cable is a convex function, if $\widehat{\tau}_a$ and $\widehat{\tau}_b$ use the same curve parameter.
\end{lemma}
\begin{proof}
    Because both $\widehat{\tau}_a$ and $\widehat{\tau}_b$ are the shortest paths in their respective homotopy classes,
    we can assert that they comprise sequences of pairwise connected straight line motions which lie on the RVG edges (following Corollary~\ref{thm:shortest-path-on-rvg}).
    Moreover, each straight line motion is of one of the following type:
    \begin{itemize}
        \item[\textbf{F}:] the motion is on the same edge as the cable and is [F]ollowing the cable, or
        \item[\textbf{L}:] the motion is on the same edge as the cable and is [L]eading the cable, or
        \item[\textbf{O}:] is any [O]ther straight line motion.
    \end{itemize}
    
    We will argue that the length of cable consumed as a robot moves along such trajectories is a convex function by showing that the gradient of the function is monotonically increasing.
    A technical difficulty with cable consumption is that it is a continuous but only piece-wise differentiable function.
    In circumstances where the derivative is undefined, we take the value of the derivative from the right.
    These circumstances arise when contacts between the cable and obstacles are made or broken. 
    The cable consumption is a function of two curve parameters:
    so by derivative we are referring to the partials with respect to each parameter---one for each robot.
    
\begin{figure*}[t]
    \centering
    \begin{subfigure}[t]{\halfpage}
        \centering
        \includegraphics[scale=1.8]{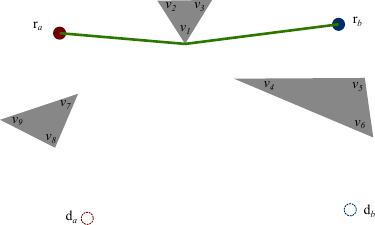}
        \caption{The original configuration of the cable.}
    \end{subfigure} \hfill
    \begin{subfigure}[t]{\halfpage}
        \centering
        \includegraphics[scale=1.8]{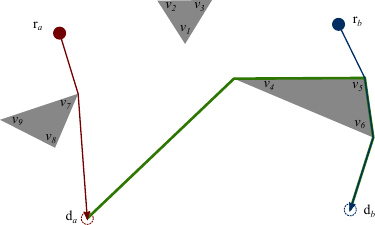}
        \caption{The optimal solution to the given \trpmpp which is found by searching the search tree.}
    \end{subfigure}
    \caption{An example scenario for which we illustrate part of the search tree in Fig.~\ref{fig:planning-tree-details}}
    \label{fig:planning-tree}
\end{figure*}

    When the robot is moving on an \textbf{F} (or \textbf{L}) segment the derivative is $-1$ (or $1$, respectively).
    For individual \textbf{O} segments the function might have multiple pieces.
    For each piece of each such segment, the pieces are regions where the points that the cable contacts are unchanged.
    Then, with a single robot moving, only the cable between the robot and the closest contact point contributes to the change in cable consumption. 
    Writing the consumption as a function of the curve parameter and simply taking its derivative results in a monotonically increasing function
    whose value is in the $(-1, 1)$ interval.
    When unwinding around an object the derivative is negative but increasing;
    when breaking contact, the radius to the closest object increases, and so does the derivative.
    When making or maintaining contact, the derivative is positive, but still increasing. It is, thus, always increasing.
    (Fig.~\ref{fig:convexity} illustrates this fact using an example.)
    
    Next, we argue that $\widehat{\tau}_a$ and $\widehat{\tau}_b$ have a specific structure:
    each trajectory is of the form $\textbf{F}^*\textbf{O}^*\textbf{L}^*$ (where we have used regular expression notation with Kleene stars).
    Showing the following suffices:
    \begin{enumerate}[(a)]
        \item no \textbf{O} segment is followed by an \textbf{F} segment, and
        \item no \textbf{L} segment is followed by an \textbf{O} segment, and 
        \item  no \textbf{L} segment is followed by an \textbf{F} segment. 
    \end{enumerate}
   
    For proof of (a): Assume there exists a motion,
    $\widehat{\tau}_1$, in which segment $A$ of type \textbf{O} ending at vertex $v$
    is followed by an \textbf{F} segment, starting at $v$.
    As we may assume the cable is always taut, we can assert that it lies on the shortest path to $v$.
    Thus, replacing $A$ with a segment/segments in which the robot follows the cable to arrive at $v$ yields a shorter path than $\widehat{\tau}_1$.
    This is a contradiction as $\widehat{\tau}_1$ is the shortest path in its homotopy class.
    
    Similarly, proof of (b): Suppose there exists a motion,
    $\widehat{\tau}_2$, in which an \textbf{L} segment ending at vertex $v$
    is followed by segment $B$ of type \textbf{O} starting at $v$ and ending at $v'$.
    Given that the robot is leading the cable and that we may assume the cable is always taut, the cable lies on the shortest path from $v$ to $v'$.
    Again, replacing $B$ with a segment/segments in which the robot follows the path of a taut cable to arrive at $v'$ will be shorter than $\widehat{\tau}_2$.
    But then $\widehat{\tau}_2$ can't have minimal length within its homotopy class.

    The proof for the impossibility of (c) is as follows: Tautening such a cable leads to a shorter, pure \textbf{F} segment. 
    The \textbf{L} segment provided excess length; and its removal
    reduces the distance along which the robots need to move. 
    But this contradicts the fact that the paths are the shortest in their respective homotopy classes.

    Taken together, this proves that $\widehat{\tau}_a$ and $\widehat{\tau}_b$ are of the form $\textbf{F}^*\textbf{O}^*\textbf{L}^*$.
    Hence, global monotonicity of the partial derivative of the cable consumption function holds if we can show that monotonicity is preserved between two consecutive \textbf{O} segments.
    To show no violation occurs at the transition between segments, monotonicity in a small open interval around the transition is sufficient.
    If the two \textbf{O}-segments, $g_1$ and $g_2$,  are collinear,
    then they could be treated as a single segment and the prior argument for a single \textbf{O}-segment holds. 
    Hence, there must be a `turn' from segment $g_1$ to $g_2$.
    Both segments are on the RVG so that turn occurs at a vertex $v_o$ of some obstacle.
    Presume that we extend and continue along $g_1$ an extra $\epsilon>0$;
    then the derivative of the cable consumption continues to increase (as the single segment argument holds).
    Since $g_2$ is not along this little extension, it falls to one side.
    If that side is away from the cable-obstacle contact,
    then the additional motion away consumes extra cable,
    so the derivative only increases faster.
    Otherwise, when turning towards the contact, monotonicity may indeed fail.
    However, such a turn leads to a contradiction;
    two cases are possible: the obstacle to which $v_o$ belongs is on the inside of the turn, or it is on the outside.
    If it is on the inside, then the cable itself wraps around $v_o$,
    in which case $g_2$ is not an \textbf{O} segment (but an \textbf{L} one).
    If the obstacle is on the outside, then simply shaving off a small corner at the turn is feasible.
    But that is shorter, contradicting the supposition that $g_1$ and $g_2$ result from shortest motions in their homotopy class.
    
    Hence monotonicity of the partial derivative of the cable consumption holds across the entire $\unitInterval$.
    Now consider the concurrent motion of both robots:
    we feed them the same curve parameter and  the derivative of the cable consumption simply becomes the total derivative.
    Being the sum of two partials, each of which is monotone, gives a monotone function.
    Thus, the total cable consumption is a convex upward function of the curve parameter.
\end{proof}

The preceding does the heavy lifting for the proof of our theorem.

\begin{theorem}\label{thm:rvg-related-solution}
    If $(\tau_a, \tau_b)$ is a solution for a given \trpmpp, then
        $\exists (\tau^{VG}_a, \tau^{VG}_b) \in \Pi_a \times \Pi_b:
        \cStar(\tau^{VG}_a, \tau^{VG}_b, c_0) \leq  \ell$.
\end{theorem}
\begin{proof}
    Let $\widehat{\tau}_a$ and $\widehat{\tau}_b$ be the shortest paths homotopic to $\tau_a$ and $\tau_b$, respectively.
    Hence, $(\widehat{\tau}_a, \widehat{\tau}_b) \in \Pi_a \times \Pi_b$ from Corollary~\ref{thm:shortest-path-on-rvg}.
    Then to prove this theorem it suffices to show that the following inequality holds:
    $$
    \cStar(\widehat{\tau}_a, \widehat{\tau}_b, c_0) \leq \ell \text{.}
    $$
    
    Since \,$\hat{\cdot}$\, is homotopy preserving, the following conditions hold:
    \begin{align*}
        &\length{{\cathat{0}{\widehat{\tau}_a}{\widehat{\tau}_b}{c_0}}} = \length{{\cathat{0}{\tau_a}{\tau_b}{c_0}}} = \length{\widehat{c}_0} \text{,}\\
        &\length{{\cathat{1}{\widehat{\tau}_a}{\widehat{\tau}_b}{c_0}}} = \length{{\cathat{1}{\tau_a}{\tau_b}{c_0}}}\text{.}
    \end{align*}
    Moreover, because $(\tau_a, \tau_b)$ is a solution we have\\
    $\length{{\cathat{0}{\widehat{\tau}_a}{\widehat{\tau}_b}{c_0}}} \leq \ell$, and\\
    $\length{{\cathat{1}{\widehat{\tau}_a}{\widehat{\tau}_b}{c_0}}} \leq \ell$.

    Following Lemma~\ref{thm:rvg-paths-are-convex},
    feeding the same curve parameter to both $\widehat{\tau}_a$ and $\widehat{\tau}_b$,
    gives a convex upward cable consumption.
    Therefore,
    \begin{align*}
        \max_{t \in \unitInterval}\,&\length{{\cathat{t}{\widehat{\tau}_a}{\widehat{\tau}_b}{c}}} =\\
        &\max_{t \in \{0, 1\}}\,\length{{\cathat{t}{\widehat{\tau}_a }{\widehat{\tau}_b}{c}}} \leq \ell \text{.}
    \end{align*}

    Because the above holds for the trivial identity re-parameterization, the minimum over the set of all re-parameterization pairs must be less than or equal to the above.
    This fact, when combined with Lemma~\ref{thm:cstar-is-less-than-ell}, completes the proof.
\end{proof}

\begin{corollary}[Distance Optimality]\label{thm:dist-optimality}
    Let $(\tau^\star_a, \tau^\star_b)$ be a distance optimal solution to a given \trpmpp. Then
    the RVG also contains a distance optimal solution, \\
    $(\tau^{\star VG}_a, \tau^{\star VG}_b)$.
\end{corollary}
\begin{proof}
    A pair $(\tau^{\star VG}_a, \tau^{\star VG}_b) \in
    \Pi_a \times \Pi_b$ is constructed from 
    $(\tau^\star_a, \tau^\star_b)$.
    Simply take 
    $(\tau^{\star VG}_a, \tau^{\star VG}_b) 
    =(\widehat{\tau^\star_a}, \widehat{\tau^\star_b})$ and
    observe the three following facts:
    \begin{itemize}
        \item $\widehat{\tau^\star_a}$ and $\widehat{\tau^\star_b}$ is on the RVG,
        \item the distance traveled along $\widehat{\tau^\star_a}$ and $\widehat{\tau^\star_b}$ for each robot is no worse following Corollary~\ref{thm:shortest-path-on-rvg},
        \item $\cStar(\widehat{\tau^\star_a}, \widehat{\tau^\star_b}, c_0) \leq \ell$ following Theorem~\ref{thm:rvg-related-solution}.
    \end{itemize}
\end{proof}

\section{The Efficient Planning Algorithm} \label{sec:planning-algorithm}

\begin{figure*}[t]
    \begin{minipage}{0.33\textwidth}
        \begin{subfigure}[t]{\textwidth}
            \centering
            \includegraphics[scale=1.0]{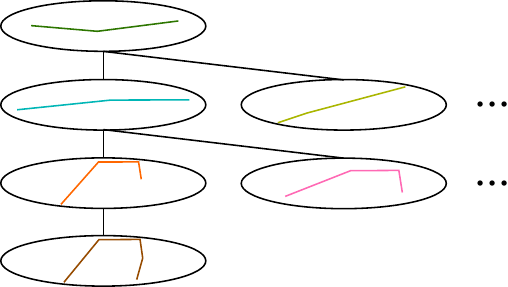}
            \caption{A visual representation of the search tree, with vertices  bearing labels that are iconic depictions of the cable configurations; see (b)--(d).}
        \end{subfigure}
    \end{minipage}
    \hfill
    \begin{minipage}{0.65\textwidth}
        \begin{subfigure}[t]{0.3\textwidth}
            \centering
            \includegraphics[scale=0.8]{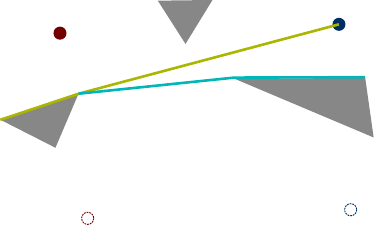}
            \caption{Two of the configurations that can be achieved from the configuration stored in the root node.}
        \end{subfigure}
        \hfill
        \begin{subfigure}[t]{0.3\textwidth}
            \centering
            \includegraphics[scale=0.8]{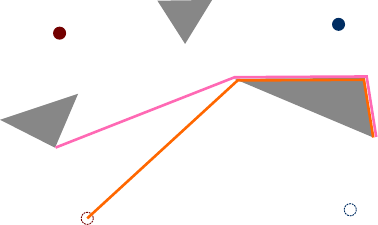}
            \caption{
                Both cables shown here can be achieved from the cyan one in the previous snapshot.
                Orange is preferred as it reduces the distance traveled by robot $a$.
            }
        \end{subfigure}
        \hfill
        \begin{subfigure}[t]{0.3\textwidth}
            \centering
            \includegraphics[scale=0.8]{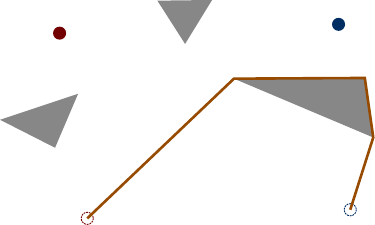}
            \caption{The final configuration to arrive at the destination node.}
        \end{subfigure}
    \end{minipage}
    \caption{
        A small portion of the search tree for the \trpmpp in Fig.~\ref{fig:planning-tree}.
        The optimal solution to the given \trpmpp in this figure can be obtained by traversing the node containing the brown cable configuration up to the root.
    }
    \label{fig:planning-tree-details}
\end{figure*}

In this section we present an implementation of \astar \citep{Hart1968-astar} to construct and explore the search tree for the solution to this problem.
Consider a \trpmpp with
\trpmppTuple.
The search tree's nodes represent cable configurations.
An edge between two nodes represents a motion for each robot along an RVG edge that abides by the cable requirements.
To do so, we define a data-structure with the following fields:
\begin{itemize}
    \item \emph{taut cable}: a list of RVG vertices each of which is visible from the previous ones describing a valid cable configuration,
    \item \emph{cost}: a pair of costs indicating the distance traveled by each robot to arrive at this cable configuration from their initial poses,
    \item a reference to a parent node.
\end{itemize}

Fig.~\ref{fig:planning-tree}
provides an example scenario, and
Fig.~\ref{fig:planning-tree-details} gives the structure of the search tree for this scenario.
Using \astar search, Algorithm~\ref{alg:search} explores the search tree\,---which gives a structured representation to the set $\Pi_a \times \Pi_b$---\,for an optimal solution to a given \trpmpp.
Being a type of informed search, \astar uses an estimated cost computed as the sum of a cost function and an heuristic function.
The cost is calculated cumulatively with a constant time operation when creating/updating a node.
When using an admissible heuristic function, we can safely terminate the search in a branch whose estimated cost is higher than the best solution found~\citep{Russell2010-artificial-intelligence}.
In Sec.~\ref{sec:disscuss-informed-search} we examine the effect of different admissible heuristics on the efficiency of the algorithm.

\begin{algorithm}
	\caption{\astar over Cable Configurations}
	\label{alg:search}
	\begin{algorithmic}[1]
		\STATE \textbf{Search}\trpmppTuple
        \STATE Build RVG, $g$, with vertices $\{\verts{O}, \pt{r_a},\pt{r_b}, \pt{d_a}, \pt{d_b}\}$
		\STATE root $=$ Node($\widehat{c}_0$, $\varnothing$)
		\STATE priorityQ $= \{ $ root $ \}$
		\WHILE{priorityQ is not empty}
            \STATE n = priorityQ.dequeue()
            \IF{n.cable.first $= \pt{d_a}$ \AND n.cable.last $= \pt{d_b}$}\label{alg:destination-condition-line}
                \IF{$\length{\text{n.cable}} \leq \ell$}\label{alg:termination-condition-line}
                    \RETURN constructed path by following parent references from~n up to the root
                \ELSE
                    \STATE \textbf{continue}
                \ENDIF
            \ENDIF
            \STATE $V_a = g.$visibleVerts$($n.cable.first$)$
            \STATE $V_b = g.$visibleVerts$($n.cable.last$)$
            \FORALL{$(v_a, v_b) \in V_a \times V_b$}
                \STATE $c_1 = [v_a] +$ n.cable $+ [v_b]$
                \STATE child $=$ Node($\widehat{c}_1$, n)
                \STATE priorityQ.enqueue(child)
            \ENDFOR
		\ENDWHILE
	\end{algorithmic}
\end{algorithm}

Given any node in the tree, we can obtain a pair of paths\,---one for each robot---\,that continuously transforms the original cable configuration to the configuration stored in the node.
To do so, we can traverse the tree from the root down to the node.
Start with $\pi_a$ and $\pi_b$ as two empty sequences.
While traversing the tree down to the node, for each node we take the first and last element of the taut tether and append it to the end of $\pi_a$ and $\pi_b$, respectively.
Notice that the sequences $\pi_a$ and $\pi_b$ will have equal cardinality.
In fact, we see that the \emph{taut cable} list at each vertex operates like a deque. Each robot, by making continuous motions
in the plane, can modify the cable. The important qualitative changes to the cable are cable events: wherein an obstacle vertex is added or deleted. The physical property of the cable means that each robot's motion can only push or pop at its respective end of the cable.

\begin{theorem}[Soundness]\label{thm:soundness}
    A pair of paths, $(\tau_a, \tau_b)$, generated by traversing the above mentioned search tree from the root to a leaf
    is a solution to the corresponding \trpmpp,
    if the leaf contains a taut tether configuration whose end points are $\pt{d_a}$ and $\pt{d_b}$.
\end{theorem}
\begin{proof}
    The root contains the taut tether $\hat{c}_0$, which by definition has end points $\pt{r_a}$ and $\pt{r_b}$.
    Because the algorithm checks whether a leaf contains a taut tether configuration whose end points are $\pt{d_a}$ and $\pt{d_b}$ on line~\ref{alg:destination-condition-line},
    the requirements $ \tau_a(0) = \pt{r_a} $, $ \tau_a(1) = \pt{d_a} $, $ \tau_a(0) = \pt{r_a} $, and $ \tau_a(1) = \pt{d_a} $ is satisfied.
    Using the result of Theorem~\ref{thm:rvg-related-solution}, the check on line~\ref{alg:termination-condition-line} ensures $(\tau_a, \tau_b)$
    satisfies the final requirement for a solution.
\end{proof}

\begin{theorem}
    Algorithm~\ref{alg:search} is complete.
\end{theorem}
\begin{proof}
    Completeness is shown by considering completeness of \astar search, and Corollary~\ref{thm:dist-optimality} that indicates that the appropriate structure generated from the RVG is the solution space.
\end{proof}

\section{From Paths to Trajectories} \label{sec:optimal-control}
\label{sec:times}

In Section~\ref{sec:problem-stat} we defined a \emph{distance optimal} solution. 
Solving a \trpmpp only requires the existence of a \wfeas. This abstracts away details of time and allows us to consider a geometric problem, i.e., one on paths. But actually moving the robots requires trajectories. This last step is important to solve the problem we have laid out: we have proven the existence of some timing for the paths returned by the algorithm, but some timings (indeed many of them!) can violate the cable constraint during execution. Thus, next,
 we prescribe a suitable timing.

Informally, an execution for a tethered pair of robots is a function that maps a \trpmpp solution (\textit{i.e.}, a pair of paths) to a pair of trajectories (as defined by \citet{Latombe1991}),
one for each robot.
We assume robots $a$ and $b$ both have maximum speeds of $mv$ and that 
turning is instantaneous.
For simplicity, we further assume that the robots can achieve speed $mv$ immediately.

For a given a pair of paths $(\tau_a, \tau_b)$, define 
\begin{equation}
\qquad T =  \max\left\{ \frac{\length{\tau_a}}{mv}, \frac{\length{\tau_b}}{mv}\right\},
\tag{\mbox{$\dagger$}}
\label{eq:trajectoryT}
\end{equation}
that is, $T$ is the time which it takes for the robot
that, moving at maximum speed throughout, arrives at its destination second. Then we construct trajectories  $\Upsilon_a:[0,T] \to \freespace$ and $\Upsilon_b:[0,T] \to \freespace$ 
as follows
\begin{equation}
\quad\Upsilon_a(\lambda) = \tau_a\left(\frac{\lambda}{T}\right) \quad\text{ and }\quad
\Upsilon_b(\lambda) = \tau_b\left(\frac{\lambda}{T}\right).
\tag{\mbox{$\ddagger$}}
\label{eq:trajectory}
\end{equation}

\begin{figure*}[t]
    \centering
    \begin{subfigure}[t]{\quarterpage}
        \centering
        \includegraphics[scale=0.7]{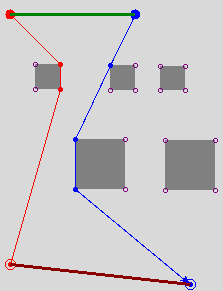}
        \caption{$\ell = 200$ and $\ell = 300$}\label{fig:discussion-diff-length-300}
    \end{subfigure}\hfill
    \begin{subfigure}[t]{\quarterpage}
        \centering
        \includegraphics[scale=0.7]{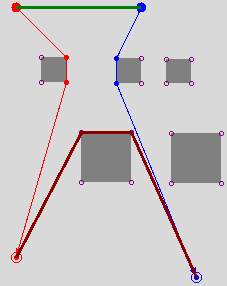}
        \caption{$\ell = 400$}
    \end{subfigure}\hfill
    \begin{subfigure}[t]{\quarterpage}
        \centering
        \includegraphics[scale=0.7]{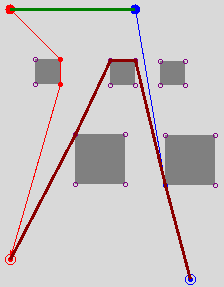}
        \caption{$\ell = 500$}
    \end{subfigure}\hfill
    \begin{subfigure}[t]{\quarterpage}
        \centering
        \includegraphics[scale=0.7]{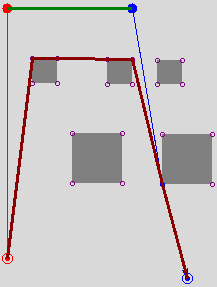}
        \caption{$\ell = 700$}
    \end{subfigure}
    \caption{Solving the same scenario with different cable length.}\label{fig:discussion-diff-length}
\end{figure*}

Note that by construction neither robot exceeds the maximum velocity.

Now,  with notation for trajectories, the time index means we can define a temporal notion of optimality.

\begin{definition}[Time optimality]\label{def:optimal-time-solution}
    A pair of trajectories $(\Upsilon^\star_a, \Upsilon^\star_b)$ that are executions 
    for a \trpmpp solution is called \emph{time optimal} if no other trajectories  
    have both $a$ and $b$ arriving at $\pt{d_a}$ and $\pt{d_b}$ sooner 
    than the later of the two under $\Upsilon^\star_a$ and $\Upsilon^\star_b$.
\end{definition}

Leading to the following final claim:

\begin{theorem}
    Given a distance optimal solution $(\tau_a, \tau_b)$ returned by Algorithm~\ref{alg:search},
    the pair of trajectories constructed by 
    \eqref{eq:trajectoryT} and \eqref{eq:trajectory} constitutes a time optimal execution.
\end{theorem}
\begin{proof}
    Algorithm~\ref{alg:search} produces distance optimal solutions such that
    $(\tau_a, \tau_b) = (\widehat{\tau}_a, \widehat{\tau}_a)$.
    Then Lemma~\ref{thm:rvg-paths-are-convex} provides that the trajectories $(t_a, t_b)$ abide by the cable constraints.
    The distance metric searches for minimum of the maximum distance traveled by either robots and the total time taken, given by \eqref{eq:trajectoryT}, is the fastest the robots could travel that distance.
\end{proof}

\section{Discussion of the Method} \label{sec:experimental-results}

\pgfplotstableread[row sep=\\,col sep=&]{
    l   & expanded   & generated  & noHV  & noHE  & hASV  & hASE  & time  & ratio\\
    200 & 290       & 1445      & 6189  & 13707 & 76    & 310   & 23.44 & 4.98\\
    300 & 290       & 1445      & 3564  & 5076  & 161   & 752   & 49.26 & 4.98\\
    400 & 41        & 450       & 0     & 0     & 34    & 352   & 9.82  & 10.97\\
    500 & 14        & 181       & 0     & 0     & 12    & 132   & 2.31  & 12.93\\
    700 & 3         & 35        & 0     & 0     & 3     & 35    & 0.29  & 11.67\\
    }\mydata

\pgfplotstableread[row sep=\\,col sep=&]{
    alg     & expanded  & generated & time  & ratio\\
    DP      & 41        & 118       & 4.52  & 2.88\\
    \astar  & 6         & 52        & 1.26  & 8.67\\
    }\dpdata

To demonstrate this algorithm, we implemented it in Python (v3).
Our implementation makes heavy use of \citet{cgal:all} for computational geometry algorithms, namely convex hull and triangulation methods.
Fig.~\ref{fig:app-gui} shows a screenshot of the user interface.

\subsection{Empirical Assessment of the Method: a Comparison} \label{sec:disscuss-dp}

\begin{figure}[ht]
    \centering
    \includegraphics[scale=0.3]{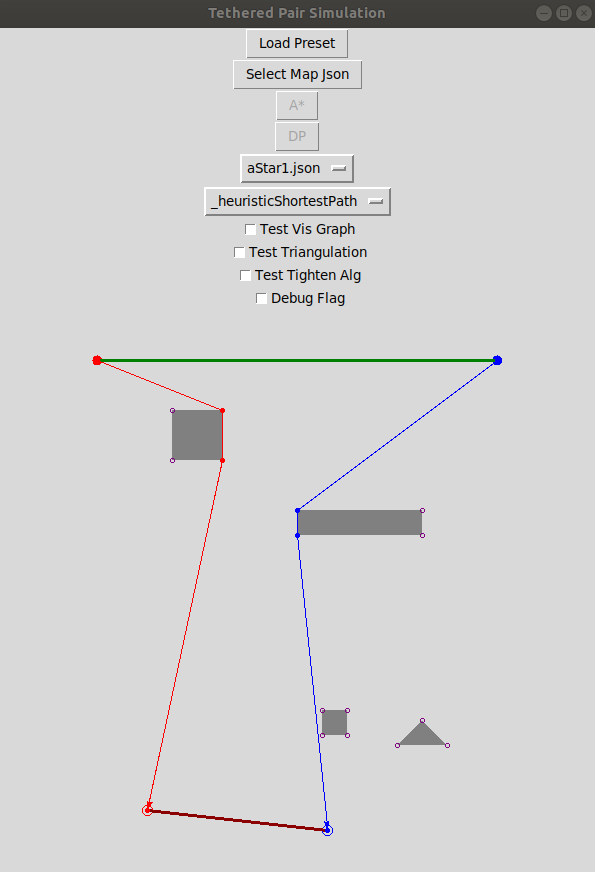}
    \caption{
        A screenshot of user interface of our implementation.
        The red and blue circle filled circles represent the robots.
        Obstacles are in grey.
        The initial and final configuration of the cable are shown in green and dark red.
        The prescribed path for each robot is shown with lines of the same color.
        The RVG vertices where for each robot is shown with small dots along the path.
    }
    \label{fig:app-gui}
\end{figure}

\begin{figure}[t]
    \centering
    \includegraphics[scale=0.7]{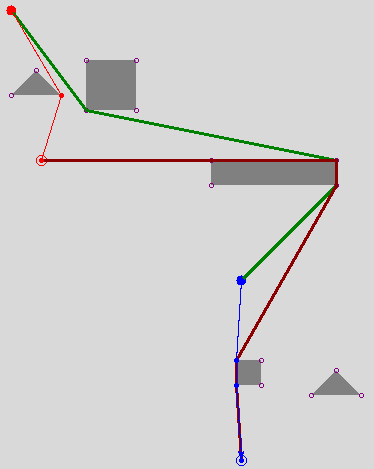}
    \caption{The scenario used for comparison between Dynamic Programming and \astar formulations.}
    \label{fig:dp-scenario}
\end{figure}

\begin{figure}[t]
    \centering
    \includegraphics[scale=0.55]{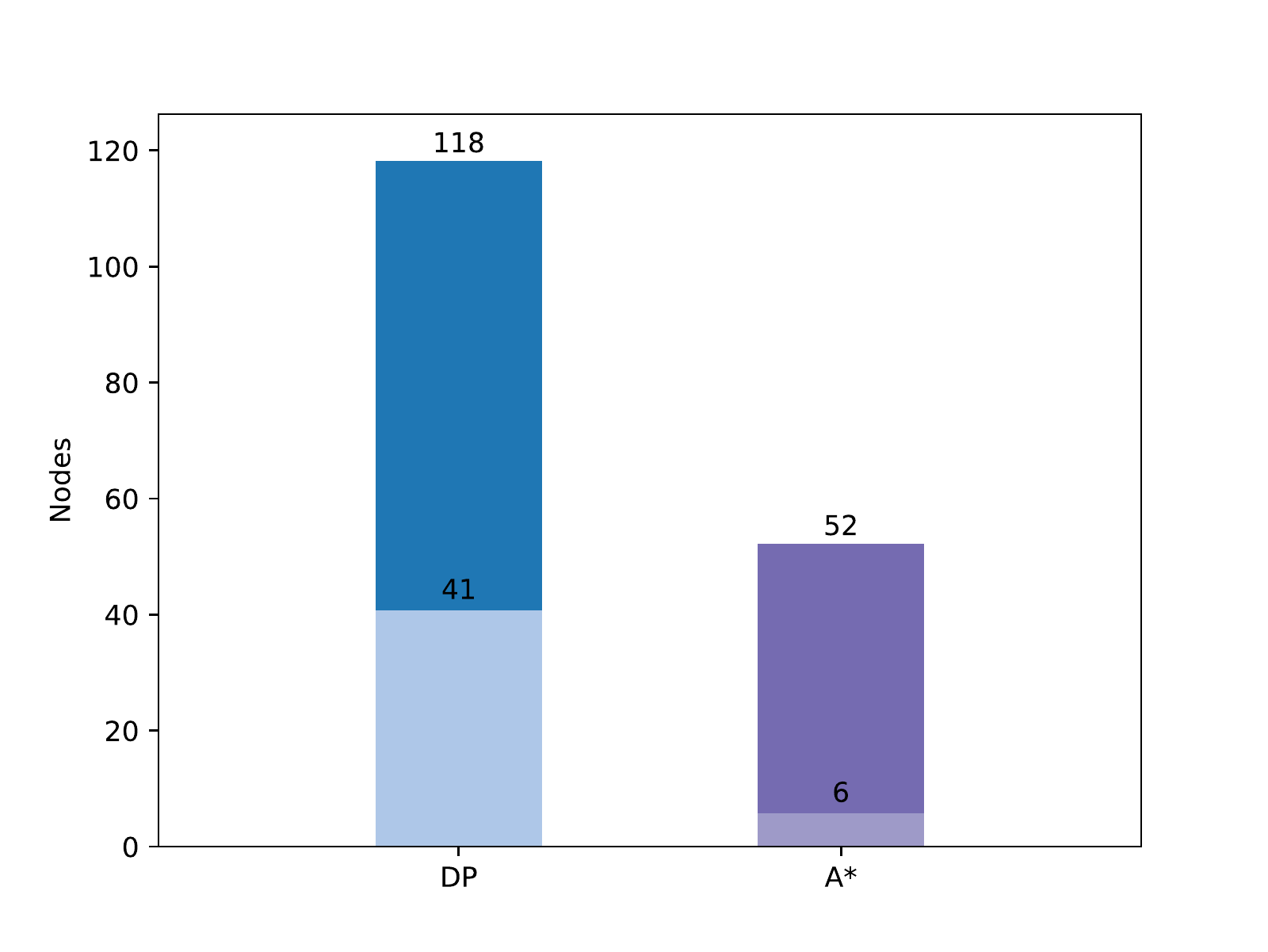}
    \caption{
        A comparison between the performance of Dynamic Programming vs. the \astar formulation.
        In each bar, the light and dark color represent the number of expanded and generated search nodes for each algorithm.
        The bars are stacked as expanded nodes are a subset of the generated ones.
    }
    \label{fig:dp-plot}
\end{figure}
    
Our earlier workshop paper \citep{teshnizi2014tethered} proposed a dynamic programming (DP) approach for motion planning for a tethered pair.
The approach uses, as a sub-procedure, a planner for the motion of a single robot connected to a fixed base with a bounded tether.
The method in that work does not rely on any specific single robot planner, as long as it can return as output the cost and the length of the tightened cable.
As stated in that paper, the DP technique only solves instances of \trpmpp of a particular form: 
it will only return motions of the robots that retain at least one contact point on the cable throughout the execution.\footnotemark~ %
Hence, it is not a complete algorithm. 

\footnotetext{Another fact is that checking whether two paths, one for each robot, will satisfy the cable length constraint when executed concurrently depends on the speeds at which they move; hence, that paper ought to have made some such assumption, as done in Section~\ref{sec:times}.}

Meaningful comparison of Algorithm~\ref{alg:search} and the DP approach requires that we select a problem for which the latter will return a solution.
The scenario shown in Fig.~\ref{fig:dp-scenario} meets this criterion and it was 
the setting used for our empirical assessment.
To remove the impact of implementation details on performance, we implemented the DP technique in Python using a specialized \trpmpp planner setup to solve for a single robot.

Here, for a heuristic function, we use the straight line (Euclidean) distance (which we denote SLD) between the robots and their destinations. This can be computed in constant time.
We measured the performance in terms of the number of search tree nodes expanded and generated as well as wall time.
Table~\ref{tab:dp} and Fig.~\ref{fig:dp-plot} present these metrics.
It is obvious that \astar formulation outperforms DP.

\begin{table}[h]
    \centering
    \caption{Comparison of the performance between DP versus \astar formulations on the scenario in Fig.~\ref{fig:dp-scenario}.}\label{tab:dp}
    \begin{tabular}{ |c|c|c|c| }
        \hline
        Algorithm & Expanded & Generated & Wall Time (\si{\second}) \\
        \hline
        DP & 41 & 118 & 4.52 \\ 
        \astar & 6 & 52 & 1.26 \\ 
        \hline
    \end{tabular}
\end{table}

\subsection{The Impact of Cable Length} \label{sec:discuss-good-bad-ugly}

\begin{figure}[t]
    \centering
    \includegraphics[scale=0.55]{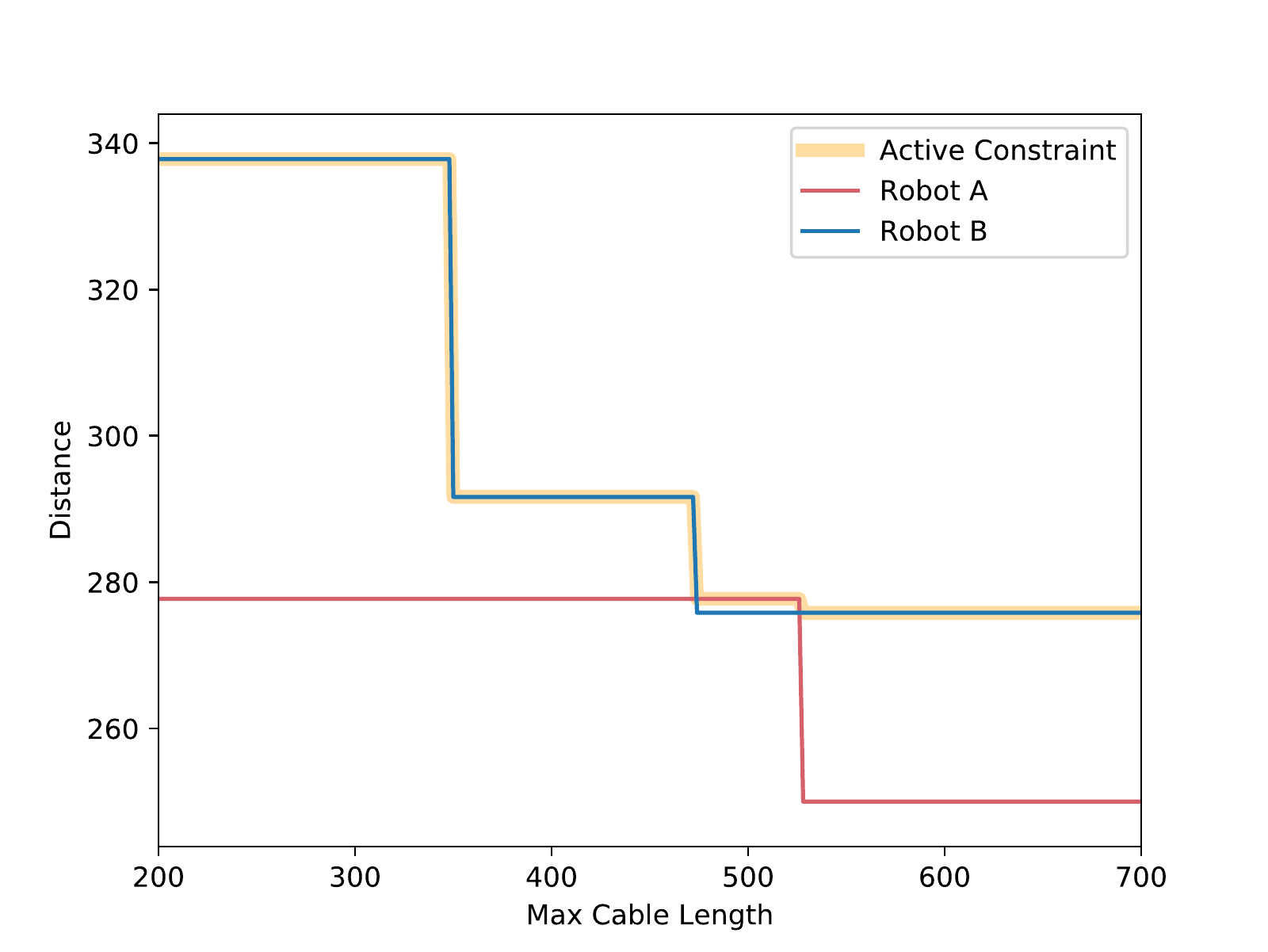}
    \caption{
        The distance traveled by each robot in optimal solution for different maximum cable lengths for scenario shown in Fig.~\ref{fig:discussion-diff-length}.
        The constraint that the planner attempts to optimize is highlight for each cable length.
    }
    \label{fig:discussion-path-length}
\end{figure}

\begin{figure}[t]
    \centering
    \includegraphics[scale=0.55]{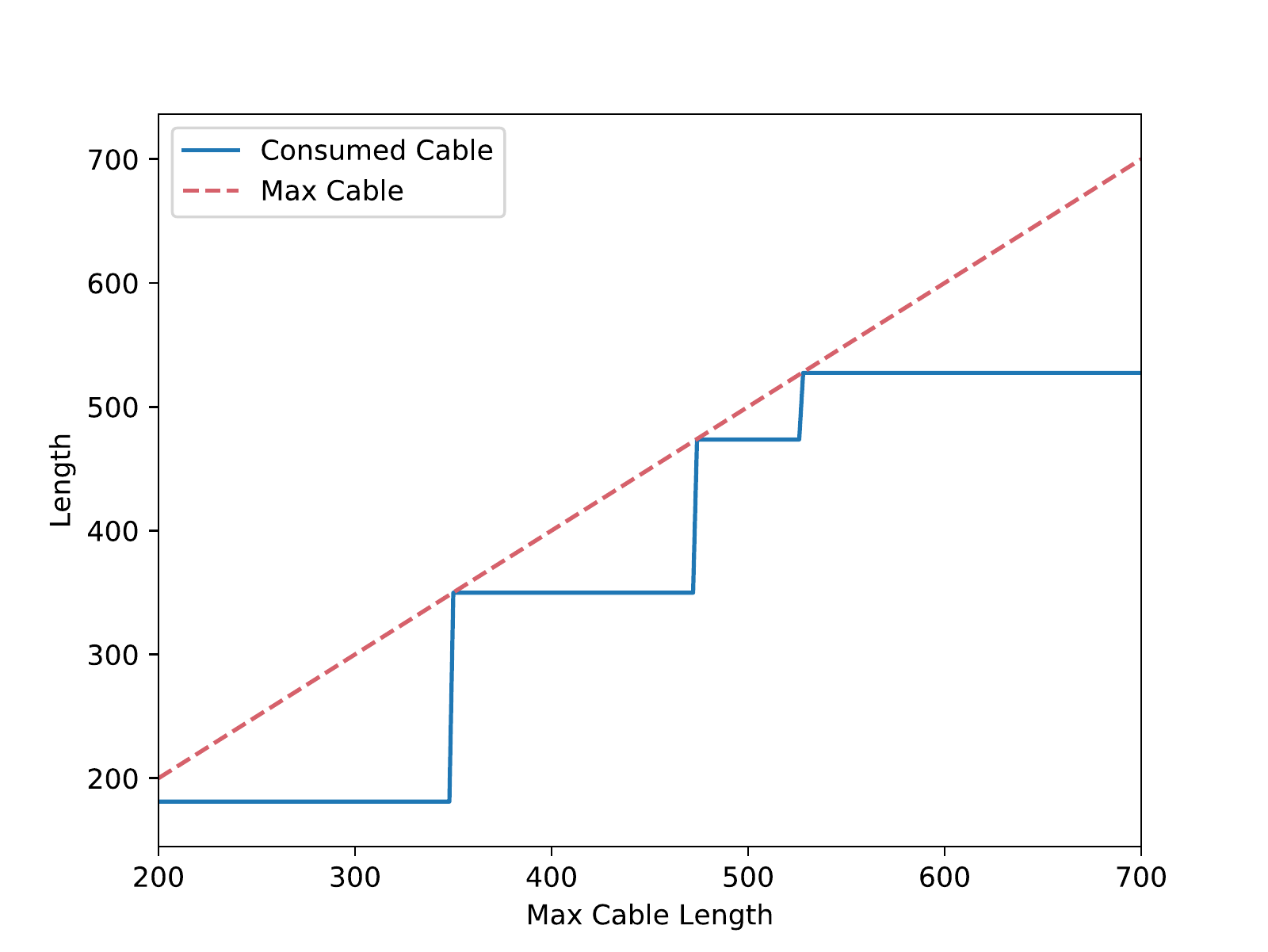}
    \caption{The length of the consumed cable in optimal solution for different maximum cable lengths for scenario shown in Fig.~\ref{fig:discussion-diff-length}}
    \label{fig:discussion-cable-consumed}
\end{figure}

\begin{figure}[t]
    \centering
    \includegraphics[scale=0.55]{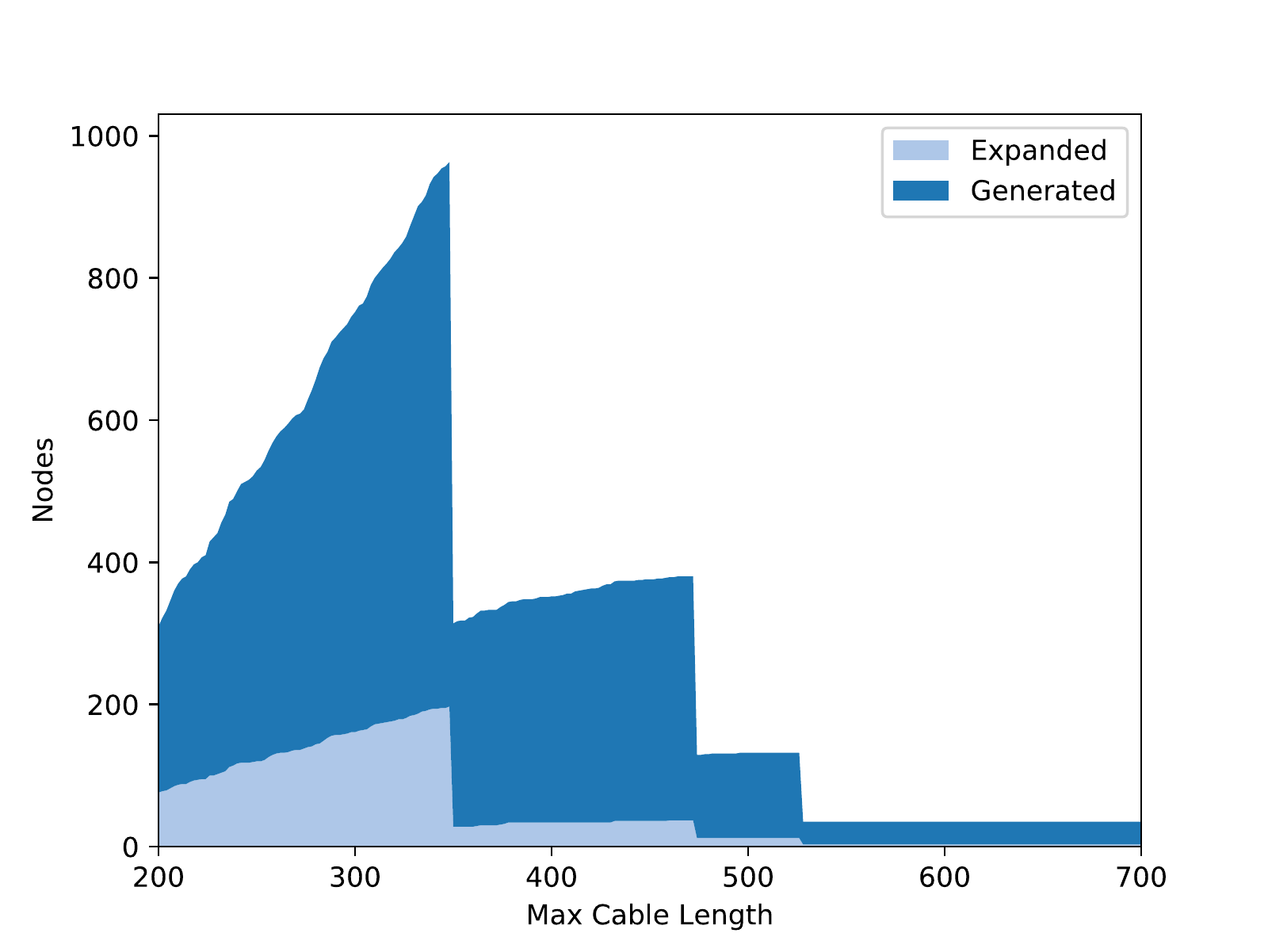}
    \caption{The number of expanded vs. generated nodes for different maximum cable lengths for scenario shown in Fig.~\ref{fig:discussion-diff-length}.}
    \label{fig:discussion-cable-length}
\end{figure}

The amount of the cable influences the search in two ways. To gain an
understanding of the impact of cable length on the problem, we solved the
sequence of planning instances depicted in
Fig.~\ref{fig:discussion-diff-length}. For each of these the environment is kept the same,
so geometrical complexity is constant,
only the maximum cable length is modified.  
Naturally, the objective values of the solutions decrease monotonically with increasing cable length. 
This can be seen in Fig.~\ref{fig:discussion-path-length}: as the length of the cable is increased, the planner minimizes the active constraint, which can be seen to switch from one robot to the other around $\ell = 470$ and again at $\ell = 525$.
The steps are qualitative changes in solutions occurring as the tether becomes long enough to
permit a new homotopy class of solutions.
Observing Fig.~\ref{fig:discussion-cable-consumed} one sees that the steps happen as the length of the consumed cable in the optimal solution equals $\ell$ exactly.

Every solution for a cable of length $\ell$ is a solution for longer cables too.
But if longer cables mean larger solutions spaces, they also mean larger search spaces.
One might expect, therefore, that solving problems with longer cables would take more time.
Fig.~\ref{fig:discussion-cable-length} shows the number of search nodes for cables ranging in length from \num{200} to \num{700}. Observe that there is a steady increase with the algorithm exploring a subset of a search space which is growing, from \num{200}--\num{350}. But then the cable becomes long enough for the search to terminate much sooner. The pattern is repeated a few times.
What happens is that the heuristic pulls the search towards parts of that space which,
though promising, are not fruitful because the topological constraint will ultimately make them dead-ends.
With a longer cable, branches that moved in the right direction but then turned out to be infeasible, now become
feasible.

\subsection{The Effect of Informed Search} \label{sec:disscuss-informed-search}

The topology of the environment as well as the cable length affects both the number of expanded and generated nodes in the search.
Making an appropriate choice of heuristic function for \astar can help.
In the previous sections, we reported results from the 
simple SLD heuristic. 
This heuristic is na\"\i ve, completely discarding topological and obstacle constraints, but still improves the efficiency of the search substantially.
We established this fact via a comparison to a base-line. We implemented Uniform Cost Search (UCS) which yields an optimal output but without any forward estimates.
Fig.~\ref{fig:discussion-uc-sld} shows the impact of the heuristic information is significant.

More sophisticated heuristics, that consider more than simple Euclidean distance, can help reduce the search time further.
Specifically, we considered two other heuristic functions, each being lesser relaxations of the problem:
\begin{itemize}
\item Shortest Path Distance (SPD) respects obstacles: it is the length of the shortest path from a robot's position vertex to its destination. We compute this by running \astar on
the single (untethered) robot problem, itself using SLD.
\item \trpmppjr respects obstacles and some topological constraints. It is computed by running a version of the problem instance at hand by but with a longer maximum cable, with SPD. 
(It is so named because it considers a diminished version of \trpmpp.)
\end{itemize}

For \trpmppjr, we sought to develop a heuristic that considers some aspects of the cable's configuration. 
The results described in previous section show that searches with longer cables tend to be
faster, so it is plausible that \trpmppjr could be computed quickly and provide high-quality
information to guide the search for the original \trpmpp.
Based on those results, it uses a cable length multiplier, as a function of the number of vertices on the taut tether, to allow faster computation for more complex cable configurations.
Notice, also, that SPD can be thought of as \trpmppjr but with $\ell=\infty$.
Hence, \trpmppjr dominates SPD, which in turn dominates SLD.

As is apparent from Fig.~\ref{fig:discussion-sld-spd-trmpp-nodes}, these heuristics do reduce the number of nodes expanded and generated by the search, but
the improvements---at least in this scenario---are minuscule.
Another important consideration is that, while 
SLD is a constant time operation, SPD and \trpmppjr usually require substantial computation.
When we compare wall times, in Fig.~\ref{fig:discussion-sld-spd-trmpp-time},
SPD offers some small improvement, whereas
\trpmppjr has no discernible value.
Both SPD and \trpmppjr involve searching relaxed problems and, in light of Valtorta's Theorem \citep{Valtorta1984}, it is perhaps unsurprising that the total number of nodes
are so similar (visible in Fig.~\ref{fig:discussion-sld-spd-trmpp-nodes}). 

\begin{figure}[t]
    \centering
    \includegraphics[scale=0.45]{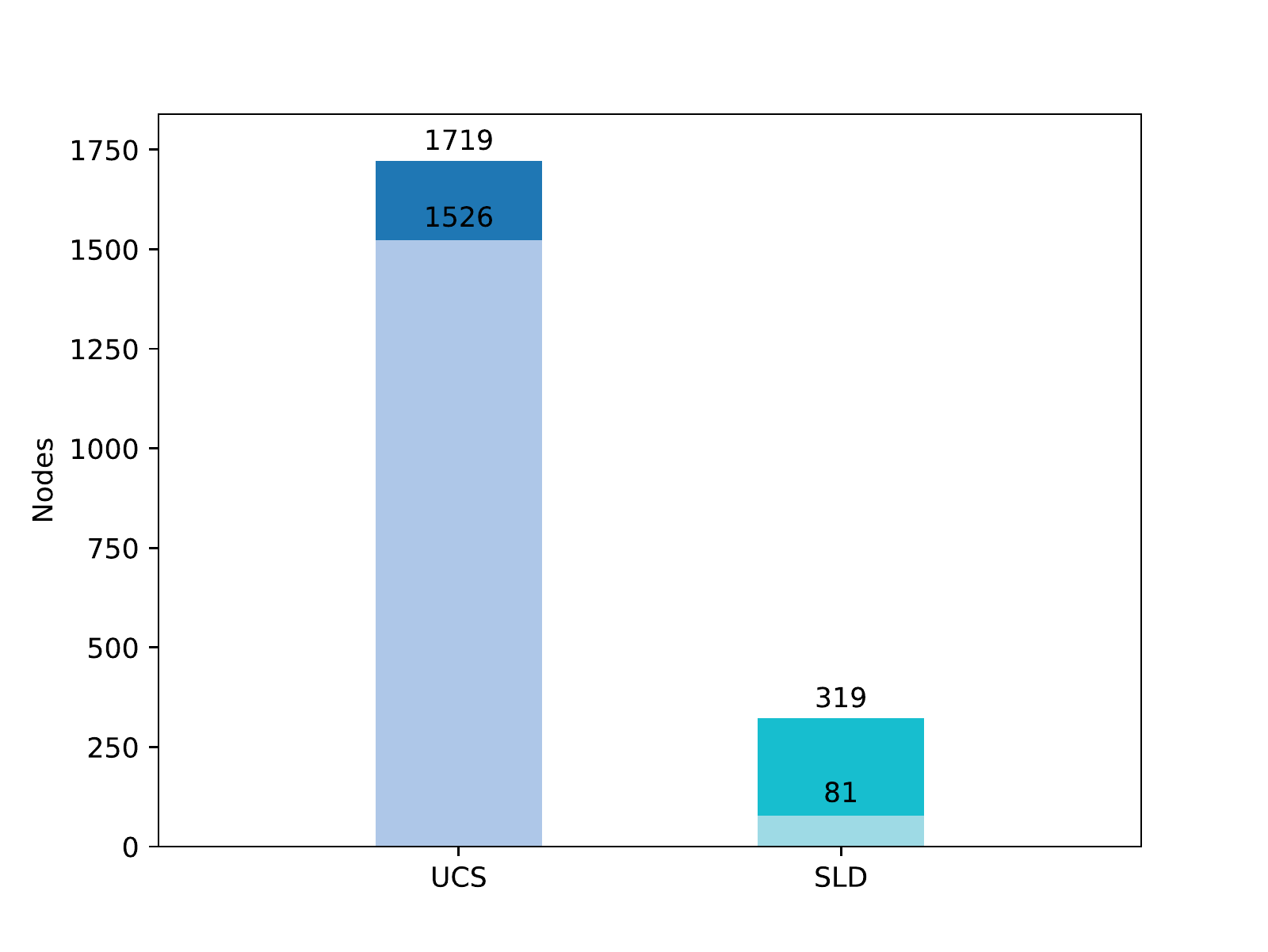}
    \caption{
        A comparison of number expanded and generated nodes between a Uniform Cost Search (UCS) vs. \astar with Straight Line Distance (SLD) heuristic.
        The light and dark colors are stacked bars, representing expanded and generated nodes respectively.
    }
    \label{fig:discussion-uc-sld}
\end{figure}

\begin{figure}[t]
    \centering
    \includegraphics[scale=0.45]{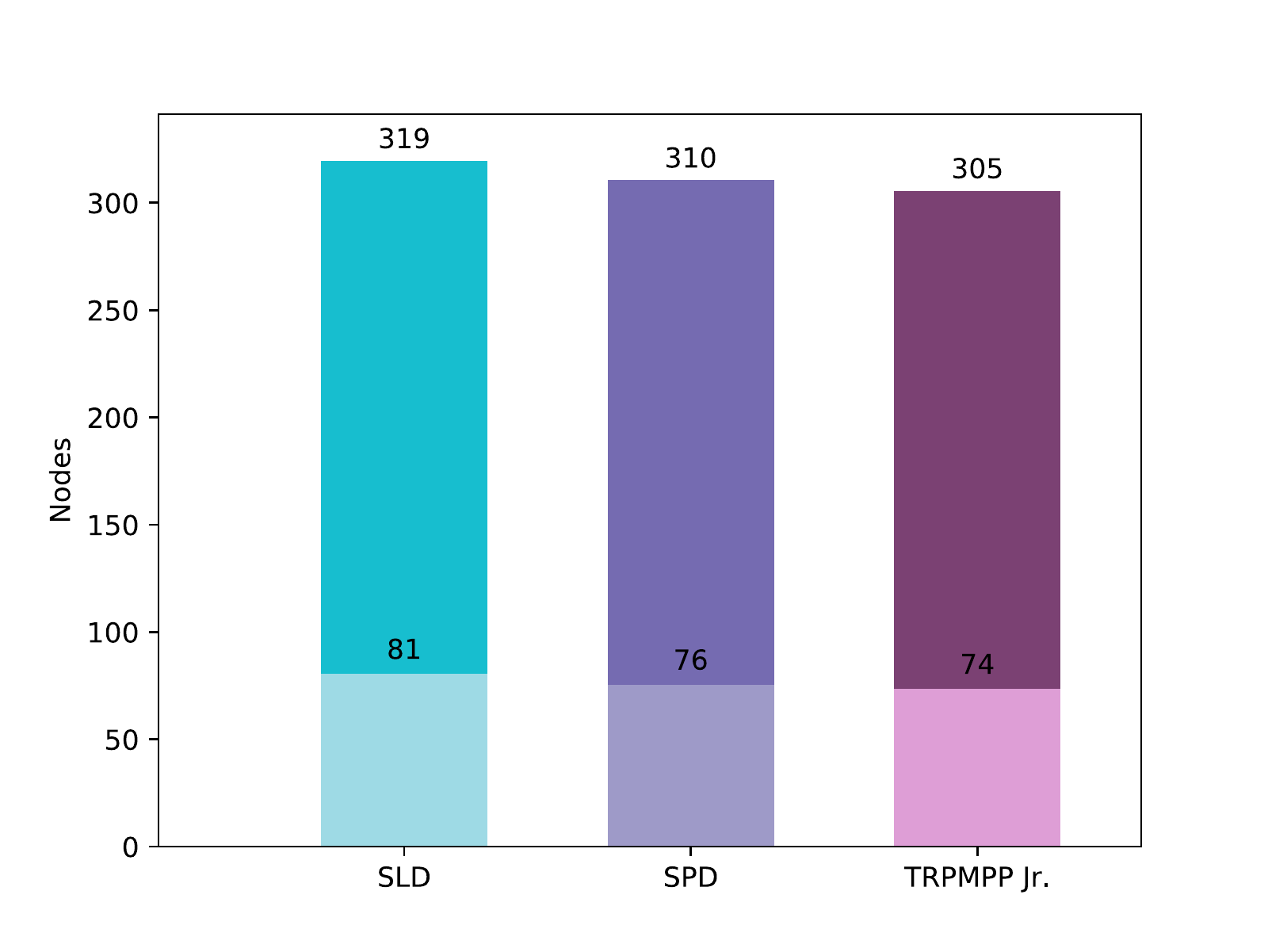}
    \caption{
        A comparison of number expanded and generated nodes between \astar with three different heuristic functions:
        Straight Line Distance (SLD), Shortest Path Distance (SPD), and \trpmppjr
        The light and dark colors represent expanded and generated nodes.
    }
    \label{fig:discussion-sld-spd-trmpp-nodes}
\end{figure}

\begin{figure}[t]
    \centering
    \includegraphics[scale=0.45]{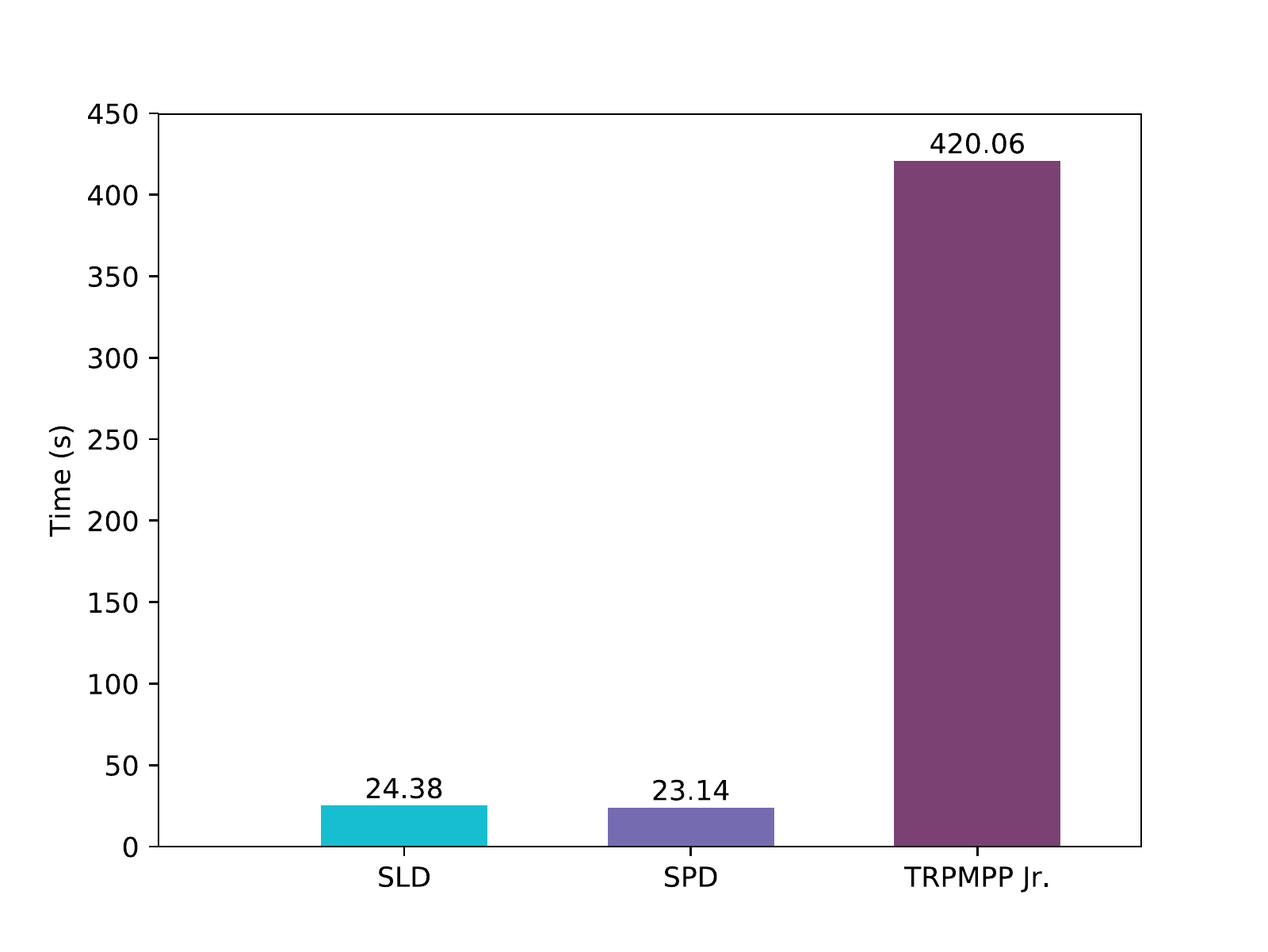}
    \caption{
        A comparison of computation wall time between \astar with three different heuristic functions:
        Straight Line Distance (SLD), Shortest Path Distance (SPD), and \trpmppjr
    }
    \label{fig:discussion-sld-spd-trmpp-time}
\end{figure}

\section{Why is the tethered robot pair problem difficult?} \label{sec:cspace-structure}

Next we further consider the relationship of the present paper with respect to \citet{teshnizi2014tethered}; that paper defined and examined a useful way of decomposing the c-space of a single tethered robot.
It constructed a tree structure that formed a skeleton that was related to an atlas representation of the c-space manifold.\footnote{Technically one must consider a manifold with a boundary, and the notion of `chart' here may be closed set; for simplicity throughout we will ignore these nuances.}

Given the inputs to the single robot problem, one visualizes its c-space most easily by sketching the charts and showing where they touch (see Fig.~\ref{fig:single-robot}).
For the single tethered robot, this is (locally) a 2 dimensional object, and consequently easy to see. 
This structure allows for the convenient and simultaneous representation of two different pieces of data: the motion of the robot through space, and the homotopy class of the cable. 

The single tethered robot problem is a special case of tethered pairs where one robot remains stationary: requiring only 2 degrees of freedom.
For a general tethered pair, a comparable figure would require four dimensions, impeding easy visualization.
In Fig.~\ref{fig:single-robot-move-base} we have moved the base of the tether slightly to the right to show the changes in the charts.
The boundaries of these charts are 2D surfaces in a 4D space.

When there is some vertex where the cable makes contact,
it is helpful to imagine two trees rooted at that point.
Fig.~\ref{fig:pair-chart-bounds} provides a  view of this.
From the perspective of our search tree, cable-obstacle contact that forms the common root is an element within the deque. (Recall the nodes of the search tree have cable configurations expressed as lists of vertices, which operate like a deque.) The subtrees on either side are represented by lists, via pushing and/or popping, that retain that element.

The robots share a total amount of cable, so as one either
consumes or gives up cable, the depth of the tree or the circular regions circumscribing the outer limits of the other robot reduce or increase, respectively.
One attractive aspect of this mental picture is how the cable describes a tree quite obviously: each time the cable wraps around an obstacle, some cable is consumed and it is as if it is anchored at that wrapping point. The recursive nature is clear because one obtains a new problem, similar to the previous one, but `smaller' in the sense of available cable.
This point of view was the main idea presented by~\citet{teshnizi-2016-tethered-paris-workshop}.
\begin{figure}[t]
    \centering
    \includegraphics[scale=0.38]{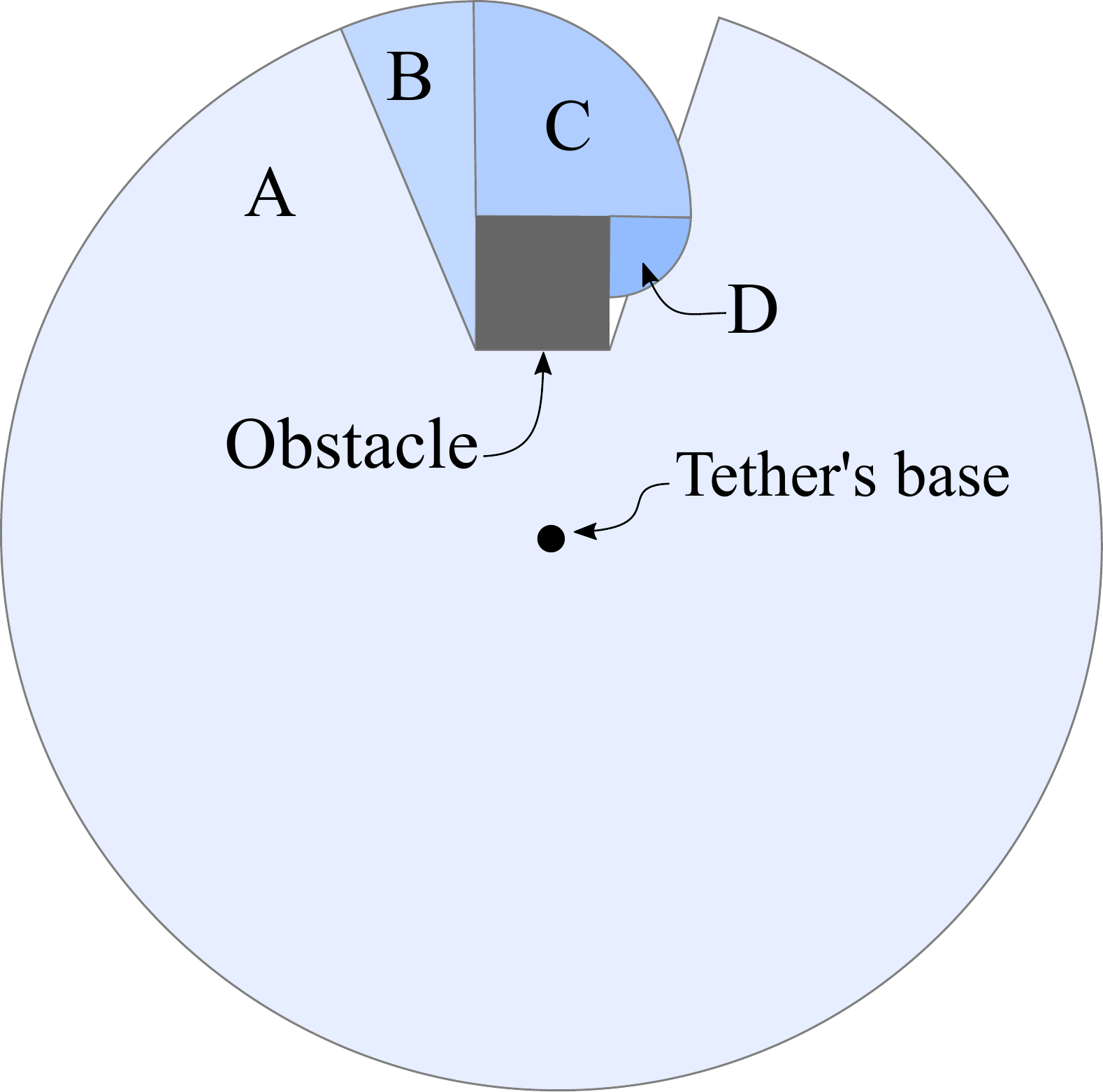}
    \caption{
        Some of the charts of the atlas representing the c-space of a tethered robot.
        $A$, $B$, $C$, and $D$ are four charts in the atlas.
        In the illustration, gray lines mark the chart boundaries.
    }
    \label{fig:single-robot}
\end{figure}

\begin{figure}[t]
    \centering
    \includegraphics[scale=0.35]{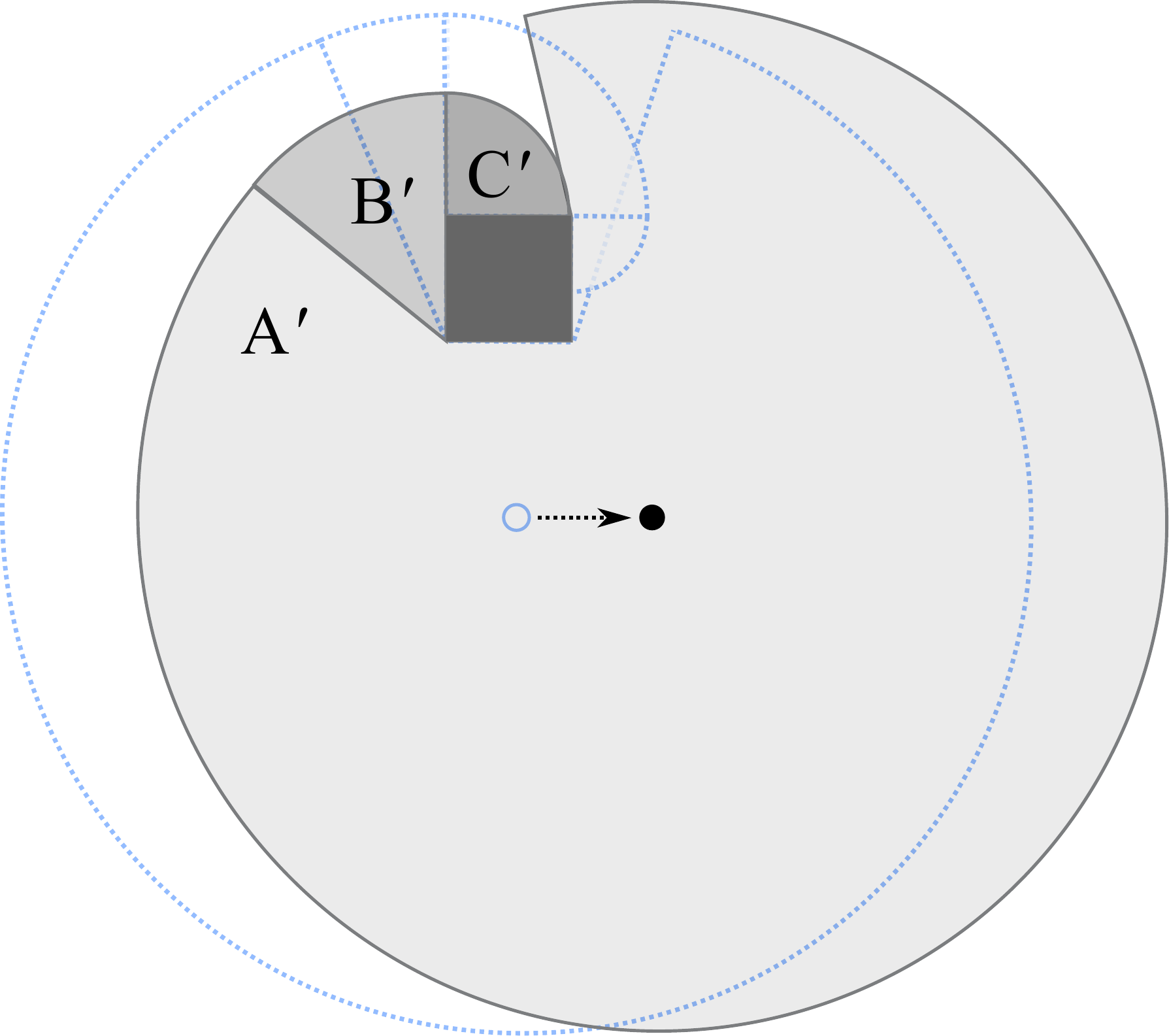}
    \caption{
        Moving the base of the tether leads to a changes in the boundaries of the charts.
        Here $A'$, $B'$, $C'$ are the new charts.
        The charts prior to moving the base are illustrated in light blue for comparison.
    }
    \label{fig:single-robot-move-base}
\end{figure}

If the robots start popping cable contacts out of the deque, moving up the tree (e.g., by following \textbf{F} segment after \textbf{F} segment) eventually the cable-obstacle contact that forms the common root is popped. At this point the node deque contains
only the two robots and physically the robots must be directly visible to one another. They can then move and, stringing the tether between them, make a contact with another obstacle vertex. Doing so, leads back to a
Fig.~\ref{fig:pair-chart-bounds}-like scenario. Hence, we can see that this top-level portion of c-space has many trees hanging like tendrils off of it. The top-level portion of c-space unlike the sort of charts made of pieces of spirals as depicted above, is truly 4D and comprises sub-regions where the pair of robots maintain line-of-sight. 
Consider that the reduced visibility graph, as a discrete structure, describes where tendrils can be found. 
But, further, the reduced visibility graph, as embedded in the plane, gives a way to think about the 4D chart if we restrict
motions to the graph edges: within this space, the robots are either on the same vertex, or one edge away from each other.
When restricted to paths on the visibility graph, within the 4D chart the robots must move on a sort of abstract boundary as they leapfrog and never get too far from one another, any deviation therefrom forming a tendril.

Previously, we said that our definition of
$\cStar$ resembles the  classic
\frechet~distance, but pointed out that it differs because $\cStar$ respects the  
homotopy class of the paths being considered.
Adopting the perspective of the c-space manifold, if 
$\tau_1$ and $\tau_2$ are curves not in 
$\freespace$, but instead in the 4D c-space, then 
$\cStar$ \emph{becomes} the \frechet~distance. Our algorithm
operates on paths in the workspace $\freespace$, which one might think of as projections from the 4D configuration space. Hence, to recover enough information to compute the distance requires
the cable configuration
 in addition to $\tau_1$ and $\tau_2$.

Prior work (such as our own~\citep{teshnizi2014tethered} and also others) has recognized the structure of the tree-like discrete skeletons that describe the atlas for a fixed tether point. Most would recognize that pushing and popping (via physical motions over critical events in terms of visibility) only occurs at the free end of the cable, so the data-structure is essentially a stack. So in summary: a 2D c-space gives a stack, a 4D c-space gives two stacks joined end-to-end, a deque.
We are not aware of a more general realization that two robots give rise to a discrete structure mirroring a deque, nor that the case of the empty structure has special significance.
One observation that is plain from this discrete data-structure point of view is that there can be multiple ways to arrive at the same 
deque contents. Indeed, this underscores the fact that the tethered pair case occurs on a manifold with cycles within it; cycles that are absent for an anchored tether.
Hence, the tendrils are woven together and one can `pick them up' to re-root from several places.

\begin{figure}[t]
    \centering
    \begin{subfigure}[t]{\halfpage}
        \centering
        \includegraphics[scale=0.6]{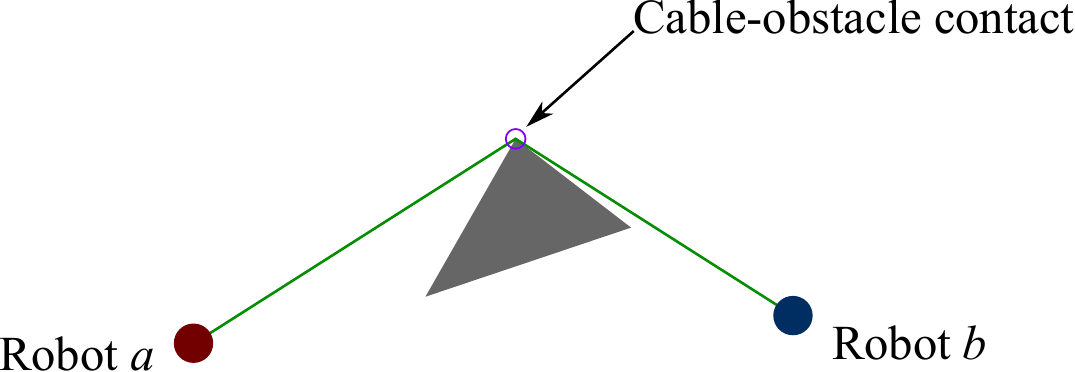}
        \caption{The configuration of the robots and their cable.}
    \end{subfigure}\hfill%
    \begin{subfigure}[t]{\quarterpage}
        \centering
        \includegraphics[scale=0.6]{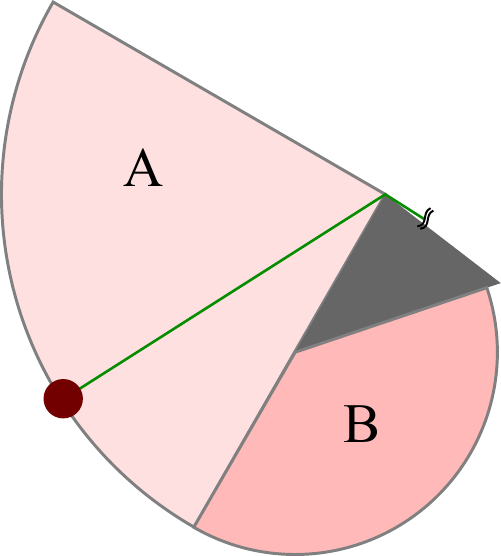}
        \caption{The visibility charts for robot $a$.}
    \end{subfigure}\hfill%
    \begin{subfigure}[t]{\quarterpage}
        \centering
        \includegraphics[scale=0.6]{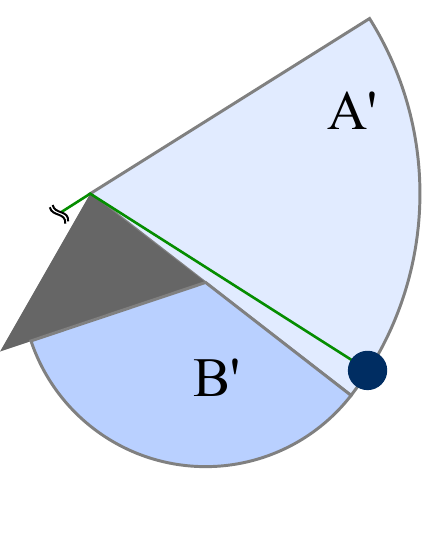}
        \caption{The visibility charts for robot $b$.}
    \end{subfigure}
    \caption{A snapshot of the conjoined atlases for a tethered robot pair after the first contact is made between the cable and an obstacle.}
    \label{fig:pair-chart-bounds}
\end{figure}

\section{Conclusion} \label{sec:conclusion}




This work presents, to the best of our knowledge, a first algorithm for planning distance optimal motions for a pair of tethered robots.
The finite cable length
makes this problem different compared to work in the literature on tether pairs:
prior work focuses on planning motions for a robot pair that need only consider topological constraints. In such cases, challenge becomes one of representing topological properties, via signatures or other techniques. 
In our case, it has turned out that the challenge was not the topological element of the problem, but rather the metric
one that arises when considering a cable of bounded length.
That is, determining if a path pair is a solution to a given \trpmpp requires ensuring existence of a feasible trajectory for the cable of bounded length.
The very question of whether 
solutions can be found in $\Pi_a \times \Pi_b$ (i.e., the RVG) is far more obvious when 
one robot's motion only constrains the other's topologically. It is not hard to see that timing (i.e., the robot's motions relative to one another) becomes coupled and the problem involves dealing with some sort of coordination.
Even glancing at the paper it is clear that greater care than is typical was needed to define what constituted a solution to a tethered pair motion planning problem so as to encode the finiteness of the cable.

Fortunately, shortest paths on an RVG help us avoid consideration of curve parameterizations: they have monotonically increasing derivative of cable consumption.
Through this property, we can consider a much reduced set of solutions without losing completeness or optimality.
Further, this provides an elegant representation of the search tree; one containing different cable configurations up to homotopy.
We start by adding the initial configuration to the root node and iteratively create the configurations.
Once we find a cable configuration that has the two robots' destinations at its ends, only then we need to check whether that configuration is permitted but the length of the cable.
Finally, using the properties of the solutions that our algorithm gives we are able to give an optimal execution. One that also minimizes the arrival time for the two robots.

The algorithm provided in this work enumerates all cable configurations that take the robots to their respective destination from least cost to most cost.
We investigated the possibility of early termination in some branches in the search tree.
However, our results have not yet provided adequate evidence that this added complexity improves performance consistently.
Our immediate future work will explore the different ways to terminate the search in a branch of the search tree.

\subsection*{Acknowledgements} 
{\small
This work was supported, in part, by the National Science Foundation through awards IIS-1453652 and IIS-1849249.
Fig.~\ref{fig:roped-team}, whose author is Cactus26, is used without any further edits under Creative Commons Attribution-Share Alike 3.0 Unported license\footnote{https://creativecommons.org/licenses/by-sa/3.0/deed.en}.
}

\bibliographystyle{spbasic}      
\bibliography{library}

\end{document}